\documentclass[onecolumn]{myclass}
\usepackage{color}
\usepackage{paralist} % for inparaenum environment
\usepackage{nicefrac} % for fractions that are rendered as a/b when necessary
\usepackage{subfigure}
\usepackage{natbib}
\usepackage{amsmath,amssymb}
\usepackage{xspace}
\usepackage{graphicx}
\usepackage{hyperref}
\usepackage{euscript}
\usepackage{wrapfig}
\usepackage{algorithm,algorithmic}

\renewcommand{\b}[1]{\ensuremath{\mathbb{#1}}}
\renewcommand{\c}[1]{\ensuremath{\EuScript{#1}}}
\newcommand{\eps}{\varepsilon}

\newcommand{\supp}{\textsc{MaxMarg}\xspace}

\newcommand{\sota}{\emph{SOTA}\xspace}
\newcommand{\sod}{\emph{SOD}\xspace}
\newcommand{\sou}{\emph{SOU}\xspace}
\newcommand{\sol}{\emph{SOL}\xspace}

\newcommand{\done}{\textsc{Data1}\xspace}
\newcommand{\dtwo}{\textsc{Data2}\xspace}
\newcommand{\dthr}{\textsc{Data3}\xspace}

\newcommand{\naiv}{\textsc{Naive}\xspace}
\newcommand{\vote}{\textsc{Voting}\xspace}
\newcommand{\rand}{\textsc{Random}\xspace}
\newcommand{\supo}{\textsc{MaxMarg}\xspace}
\newcommand{\oalgo}{\textsc{Median}\xspace}
\newcommand{\csup}{\textsc{Support}\xspace}

\newcommand{\itsupp}{\textsc{IterativeSupports}\xspace}

\numberwithin{equation}{section}
\newenvironment{packeditemize}{
\begin{itemize}
  \setlength{\itemsep}{1pt}
  \setlength{\parskip}{0pt}
  \setlength{\parsep}{0pt}
}{\end{itemize}}

\newcommand{\delete}[1] {}

\usepackage[draft,inline,nomargin,index]{fixme}
\fxsetup{theme=colorsig,mode=multiuser,inlineface=\itshape,envface=\itshape}
\FXRegisterAuthor{sv}{asv}{Suresh}
\FXRegisterAuthor{avi}{aaa}{Avishek} 
\FXRegisterAuthor{jp}{ajm}{Jeff} 

\definecolor{ForestGreen}{rgb}{0.06, 0.8, 0.06}

\definecolor{PrettyPurple}{rgb}{0.4, 0.0, 1.0}

\begin{document}

\title{Protocols for Learning Classifiers on Distributed Data}

%\author{\name Hal Daum\'e III \email hal@umiacs.umd.edu \\
%       \addr Department of Computer Science\\
%       University of Maryland\\
%       College Park, MD, USA 
%       \AND
%       \name Jeff M. Phillips \email jeffp@cs.utah.edu \\
%       \addr School of Computing\\
%       University of Utah\\
%       Salt Lake City, UT, USA
%       \AND
%       \name Avishek Saha \email avishek@cs.utah.edu \\
%       \addr School of Computing\\
%       University of Utah\\
%       Salt Lake City, UT, USA
%       \AND
%       \name Suresh Venkatasubramanian \email suresh@cs.utah.edu \\
%       \addr School of Computing\\
%       University of Utah\\
%       Salt Lake City, UT, USA}

%\editor{Leslie Pack Kaelbling}
%\editor{}

\author{Hal Daum\'e III \\ University of Maryland, CP \and Jeff M. Phillips \\ University of Utah \and Avishek Saha \\ University of Utah \and Suresh Venkatasubramanian \\ University of Utah}

\maketitle

\begin{abstract}%   <- trailing '%' for backward compatibility of .sty file
We consider the problem of learning classifiers for labeled data that has been distributed across several nodes. Our goal is to find a single classifier, with small approximation error, across all datasets while minimizing the communication between nodes. This setting models real-world communication bottlenecks in the processing of massive distributed datasets.  We present several very general sampling-based solutions as well as some two-way protocols which have a provable exponential speed-up over any one-way protocol. We focus on core problems for \emph{noiseless} data distributed across two or more nodes. The techniques we introduce are reminiscent of active learning, but rather than actively probing labels, nodes actively communicate with each other, each node simultaneously learning the important data from another node. 
\end{abstract}


%===========================================================================================================
%===========================================================================================================
\section{Introduction}
\label{sec:intro}
%\vspace{-0.25cm}

Distributed learning~\citep{langford11dist} is the study of machine learning on data distributed across
multiple locations. Examples of this setting include data gathered from sensor networks, or from data centers
located across the world, or even from different cores on a multicore architecture. In all cases, the
challenge lies in solving learning problems with minimal \emph{communication} overhead between nodes; learning algorithms cannot afford to ship all data to a central server, and must use limited communication efficiently to perform the desired tasks. 

In this paper, we introduce a framework for studying distributed classification that treats inter-node
communication as a limited resource, and present a number of algorithms for this problem that
uses inter-node interaction to reduce communication. 
% exponentially in comparison to ``one-way'' models of communication. 
Our main technique is the use of carefully chosen data and classifier descriptors that convey the most useful
information about one node to another; in that respect, our work makes use of (in spirit) the \emph{active learning} paradigm~\citep{bsettles09activesurvey}.

%Distributed classification is a typical example of this paradigm. 
For distributed classification, the dominant strategy~\citep{predd06distwireless,mcdonald10diststrucperc,mann09distmem,lazarevic01distboost}
is to design local classifiers that work well on individual nodes. These classifiers are then communicated to a central server, and then aggregation strategies like voting, averaging, or even boosting are used to compute a global classifier. 
These approaches, while designed to improve communication, do not study communication as a resource to be
used sparingly, and ignore the fact that interactions between nodes might reduce communication even further
by allowing them to learn from each others' data.  %\jeff{I think something like the follows need to be said here.  Avshek, please confirm literature: ``Furthermore, on adversarial data, these approaches lack hard guarantees on the overall error."  }

\begin{table*}[!htbp]
\centering
\begin{tabular}{cccccc} 
		{\bf Hypothesis} 	& {\bf Dimen-} & {\bf Error} & \multicolumn{2}{c}{\bf Communication Complexity} & {\bf Reference}\\
			\cline{4-5}
		{\bf Class}       & {\bf sions}  &  & {\bf Two-party} 	 & {\bf k-party} 				\\
		\hline
		\hline
		\multicolumn{6}{c}{one-way communication} \\
		\hline 
		generic    				& $d$ 	& $\eps$ & $O(\nu/\eps \log{\nu/\eps)}$ & $O(k(\nu/\eps) \log{\nu/\eps})$ 
			&  Theorem~\ref{thm:generic1way} \& \ref{thm:generickway} \\
		thresholds  			& $1$ 		& 0	& $2$                      & $2k$
			&  Lemma~\ref{lem:threshold} \& \ref{thm:k-0err} \\ 
%		intervals (Theorem~\ref{thm:intervals})   & $1$ 				& one-way 		& $4$ (constant)           &
%			&	0	& Theorem~\ref{thm:intervals} (\ref{ssec:intervals})	\\ 
		aa-rectangles  		& $d$ 			& 0 & $4d$                     & $4dk$
			& Theorem~\ref{thm:aarects}  \& \ref{thm:k-0err} \\ 
		hyperplanes  			& $d$ 				& $\eps$ & $\Omega(1/\eps)$  	     & $\Omega(k/\eps)$
			& Theorem~\ref{thm:lb-linsep} \& \ref{thm:lb-linsep-k}  \\
		%\hline
		\hline
		\multicolumn{6}{c}{two-way communication} \\
		\hline
		hyperplanes  			& $2$ 		& $\eps$ & $O(\log{1/\eps})$  	     & $O(k^2 \log{1/\eps})$
			& Theorem~\ref{thm:rounds} \& ~\ref{thm:k-lin-sep} \\
\end{tabular}
\caption{Summary of results obtained for different hypotheses classes under an \emph{adversarial}
	model with one-way and two-way communications. All results are for the \emph{noiseless} setting.
$\nu$ denotes the VC-dimension for the family of classifiers.}
%\avi{fill the missing values}}
\label{tab:contri}
%\vspace{-0.20cm}
\end{table*}
%\vspace{-0.25cm}

\paragraph{Problem definition.}
There are many aspects to formalizing the problem of learning classifiers with limited communication,
			including discussion of the data sources (i.i.d. or adversarial), data quality (noiseless or noisy),
			communication models (one-way, two-way or k-way) and classifier models (linear, non-linear, mixtures).
			In this paper, we focus on a simple core model that illustrates both the challenges and the benefits of focusing on the communication bottleneck. 

In our model, we first consider one-way and two-way communication between \emph{two} parties Alice and Bob
that receive \emph{noiseless} data sets $D_A$ and $D_B$ that result from partitioning a larger data set $D = D_A
\cup D_B$. Thereafter, we consider one-way and two-way communication between $k$ parties $P_1,P_2,\ldots,P_k$ that
receive noiseless data sets $D_1,D_2,\ldots,D_k$ partitioned from $D = \bigcup_{i=1}^k D_i$. In
either case, the partitioning may be done randomly, but might also be adversarial: indeed, a number of recent discussions~\citep{cbianchi09active-online,dekel10multoracle,laskov10ml-adv,adversariallink1,adversariallink2} highlight the need to consider adversarial data in learning scenarios.

% Our first contribution is formally defining a large family of important problems.  
% There are many settings for distributed learning, each posing its own set of challenges.  
% Data may arrive from multiple heterogeneous sources that might violate the statistical assumptions of \emph{iid} data.  In fact, recent discussions have highlighted the need to consider adversarial data in learning scenarios. 
% Moreover, the datasets partitioned in iid or adversarial fashion could be either noiseless (which implies perfect separability) or noisy (which relates to the popular paradigm of agnostic learning~\citep{balcan09agnact,dasgupta08genagact}). 
% Apart from datasets, one can additionally impose various assumptions on the communication protocol (one-way or two-way) and the classifiers learned, that is, whether all nodes learn similar (say, linear or non-linear) classifiers or a mixture thereof. 
% As can be seen, in our proposed setup, the number of possible scenarios can be
% overwhelmingly large and, in this paper, we focus on a core subset of the aforementioned landscape. 
% We consider the case of \emph{one-way} and \emph{two-way} communication between
% \emph{two-parties} (namely, $Alice (A)$ and $Bob (B)$) that receive noiseless data sets $D_A$ and $D_B$ ($D_A \cap D_B = \emptyset$) that are either \emph{iid} or \emph{adversarially partitioned}.   Some basic results also directly generalize to the noisy case.
%\jeff{and noisy idd.}
In our model, the nodes together learn (via communication) a classifier $h_k$ ($h_{AB}$ for two nodes $A$ and
		$B$) from a family of classifiers such as linear classifiers.  Let $h^*$ denote the optimal classifier
that can be learned on $D$.  Let $E_D(h)$ denote the number of points misclassified by some classifier $h$ on $D$.   
We say that $h_k$ has \emph{$\eps$-approximation error} ($\eps$-error for short) on $D$ if $E_D(h_k) -
E_D(h^*) \leq \eps |D|$.
% \hal{I'm not sure if I want to call this ``error'' or ``regret'' or
%   something else.  Standard definitions would be $\textrm{err}_D(h) =
%   E_D[h(x) \neq y]$ and regret would be $\textrm{reg}_D(h) =
%   \textrm{err}_D(h) - \min_{g \in 2^X} \textrm{err}_D(g)]$.  I would
%   lean toward calling it regret, but we really mean a regret specific
%   to the hypothesis class, so that the $\min$ ranges only over $g \in
%   \mathcal{H}$, not over all functions $g$.  I think the correct term
%   is ``external regret'' to refer to the case that you're comparing
%   against a fixed set of hypotheses, or at least that's the term from
%   online learning land.}
\emph{The goal is for $h_k$ to have at most $\eps$-error ($0 < \eps < 1$) while minimizing inter-node communication. }

In this paper, we phrase the learning task in terms of training error, rather than generalization. This is
motivated by numerous results that indicate that low training error combined with limits on the hypothesis
class used lead to good generalization bounds~\citep{bookkearns94clt}. 

%\vspace{-0.15cm}
\paragraph{Technical contributions.}
Our overall contribution, in this paper, is to model communication minimization (in distributed classification) as an active
probing problem.
We start in Section \ref{sec:random} by showing that, within our proposed framework, the one-way
communication problem can be solved trivially under i.i.d. assumptions (ref. Section~\ref{sec:random}). 
Hence, in this work, most of our effort is focused on \emph{adversarial distributions}.   
In all subsequent cases, we first help build intuition by discussing a two-party protocol and thereafter extend the
two-party results to the $k$-party case.
In Section \ref{sec:one-way} we show that, for \emph{one-way} communication, it is possible to learn optimal global classifiers \emph{exactly} (i.e., with $0$-error) for thresholds (in $\mathbb{R}^1$), intervals (in $\mathbb{R}^1$) and axis-aligned rectangles (in $\mathbb{R}^d$) with only a \emph{constant} amount of communication. 
For the case of linear separators, we prove an $\Omega(1/\eps)$ lower bound (ref. Appendix~\ref{app:lb-linsep}). 
%
%Thereafter in Section \ref{sec:linsep}, we present the main technical result of this paper
%(Theorem~\ref{thm:rounds} of Section~\ref{sec:linsep}) which shows that, for \emph{two-way} communication and
Thereafter in Section \ref{sec:linsep}, we present our \emph{two-way, two-party} communication protocol
\textsc{IterativeSupport} %\avi{need a macro for iterativesupport. some have an additional `s' -- Fix!!}
which 
%shows that, for \emph{two-way} communication and under adversarial distributions, an $\eps$-error linear separator (in $2$-d) can be 
learns an $\eps$-error classifier (under adversarial distributions)
using \emph{only} $O(\log 1/\eps)$ communication -- \emph{an exponential improvement over the one-way
	case!} Next in Section~\ref{sec:kparty}, we use the results of Section \ref{sec:linsep} to obtain an $O(k^2 \log 1/\eps)$
bound for $k$-parties using two-way communication. 
In Section~\ref{sec:expts}, we present results that demonstrate the correctness and
convergence of the linear separator algorithms and also empirically compare its performance with
a few other baselines. 
%It is non-trivial to extend our two-dimensional two-way linear separator approach to higher
%dimensions and so in this paper we limit our discussions to the initial two-dimensional case.

%
Table~\ref{tab:contri} summarizes the results obtained with references to appropriate sections of this paper.
All our results pertain to the noiseless setting which
assumes the existence of a classifier that perfectly separates the data. In Section~\ref{sec:disc}, we
provide outlines to extend our proposed results to noisy data. Finally, for cases when it is difficult to a priori ascertain the presence of noise, we present one-way communication lower bounds for learning in our model (ref. Appendix~\ref{app:lb-noisedetect}).

\section{Randomly Partitioned Distributions}
\label{sec:random}
%\vspace{-0.25cm}
We first consider the case when the data is partitioned randomly among nodes.  Specifically, each node $i$ can view its data $D_i$ as being drawn iid from $D \subset \b{R}^d$.  
We can now apply learning theory results for any family of classifiers $\c{H}$ with bounded VC-dimension $\nu$. % for $(\b{R}^d,\c{H})$.  
%First consider the non-agnostic (noiseless) setting where we assume that some classifier $h \in \c{A}$ perfectly classifies $D$.  
%We describe these results directly for the $k$-party problem since they 
Any classifier $h_S \in \c{H}$ which perfectly separates a random sample $S$ of $s = O((\nu/\eps) \log (\nu/\eps))$ samples from $D$ has at most $\eps$-classification error on $D$, with constant probability~\citep{book09anthony}.  Thus each $D_i$ can be viewed as such a sample $S$ and if $D_i$ is large enough, with no communication a node can return a classifier with small error.  

\begin{theorem}
\label{thm:id-net}
Let $\{D_1, \ldots, D_k\}$ randomly partition $D \subset \b{R}^d$.  In the noiseless setting a node $i$ can
produce a classifier from $(\b{R}^d, \c{H})$ (with VC-dimension $\nu$) with at most $\eps$-error for $\eps =
O((\nu/|D_i|) \log |D_i|)$, with constant probability.  
\end{theorem}

A similar result (with slightly worse dependence on the $D_i$) can be obtained for the noisy setting. These results indicate that the $k$-party (and hence also two-party) setting is trivial to solve if we assume random partitioning of $D$.  Thus, for the remainder of the paper we focus on protocols for adversarially  partitioned data.  

%%% Local Variables: 
%%% mode: latex
%%% TeX-master: "2012aistats-main"
%%% End: 

%===========================================================================================================
\section{\emph{One-way} Two-Party Protocols}
\label{sec:one-way}
%\vspace{-0.25cm}
We now turn to data adversarially partitioned between two nodes $A$ and $B$, as disjoint sets $D_A$ and $D_B$, respectively.  For the hypothesis classes discussed in this section, \emph{one-way protocols} where only $A$ sends data to $B$ suffices for $B$ to learn an $\eps$-error classifier.
%\paragraph{Generic One-Way Protocols.}
%\label{ssec:generic1way}
Consider first a generic setting, with $D \subset \b{R}^d$ and family of hypothesis $\c{H} \subset 2^D$ so $(\b{R}^d,\c{H})$ has VC-dimension $\nu$.  
\begin{theorem}
\label{thm:generic1way}
Assume there exists a $0$-error classifier $h^* \in \c{H}$ on $D$ where $(D,\c{H})$ has VC-dimension $\nu$.  
Then $A$ sending $s_\eps = O((\nu/\eps) \log (\nu/\eps))$ random samples ($S_A \subset D_A$) to $B$ allows $B$ to, with constant probability, produce an $\eps$-error classifier $h \in \c{H}$.
% where $(D,\c{A})$ has VC-dimension $\nu$.  %, with $\eps$-error when 
%\begin{packeditemize} \vspace{-.15in}
%\item $s_\eps = O((\nu/\eps) \log (\nu/\eps))$ in the non-agnostic setting; or
%\item $s_\eps = O(\nu/\eps^2)$ in the agnostic setting.  
%\end{packeditemize}
\end{theorem}
%\vspace{-.2in}
\begin{proof}
%In the non-agnostic setting, 
The classifier returned by $B$ will have $0$ error on $D_B \cup S_A$; thus it only has error on $D_A$.  Since $S_A$ is an $\eps$-net of $D_A$ with constant probability, then it has at most $\eps$-error on $D_A$ and hence at most $\eps$-error on $D_A \cup D_B = D$.  
%
%In the agnostic setting, the classifier is optimal on $D_B \cup S_A$ (after appropriately weighting $S_A$ so its total weight is the same as $D_A$).  Since $S_A$ is an $\eps$-approximation of $(D,\c{A})$, with constant probability, this classifier has at most an $\eps |D_A|$ difference in the error on $D_A$, as compared to $S_A$.  This also implies that the optimal classifier $h^* \in \c{A}$ on $D_A \cup D_B$ is at most $\eps |D_A|$ difference on $D_B \cup S_A$, and hence $h$ has less error than $h^*$ on $D_B \cup S_A$, and thus is at most $\eps |D_A|$ far from optimal on $D_B \cup D_A$.  
\end{proof}
%\vspace{-.15in}
A similar result with $s_\eps = O(\nu/\eps^2)$ applies to the noisy setting.  
An important technical contribution of this paper is to show that in many cases we can improve upon these general results. %We focus hereafter on specific geometric concepts. 

%--------------------------------------
\subsection{Specific Hypothesis Classes}
\label{ssec:specific}
%\subsection{Thresholds in $\b{R}^1$}
%\label{ssec:threshold}

\paragraph{Thresholds.} First we describe how to find a threshold $t \in \c{T} \subset \b{R}$ such that all points $p \in D$ with $p < t$ are positive and with $p > t$ are negative.  $A$ sends to $B$ a set $S_A$ consisting of two points in $D_A$: its largest positive point $p^+$ and its smallest negative point $p^-$.  Then $B$ returns a $0$-error classifier on $D_B \cup S_A$.
\begin{lemma}
\label{lem:threshold}
In $O(1)$ one-way communication we can find a $0$-error classifier in $(D,\c{T})$.
\end{lemma}
%\vspace{-.2in}
\begin{proof}
%$A$ sends to $B$ a set $S_A$ consisting of two points in $D_A$: its largest positive point $p^+$ and its smallest negative point $p^-$.  Then $B$ returns the perfect classifier on $D_B \cup S_A$.  
The optimal classifier $t \in \c{T}$ must lie in the range $[p^+, p^-]$ otherwise, it would misclassify some point in $D_A$, breaking our noiseless assumption.  Then any $0$-error classifier on $D_B$ within this range is has $0$ error on $D$.  
\end{proof}
%\vspace{-.15in}

%--------------------------------------
%\subsection{Intervals in $\b{R}^1$}
%\label{ssec:intervals}

\paragraph{Intervals.}We can now apply Lemma~\ref{lem:threshold} to get stronger bounds.  In particular, this generalizes to the family $\c{I}$ of intervals in $\b{R}^1$.  First $A$ finds $h_A$, its optimal classifier for $D_A$.  This interval has two end points each of which lies in between a pair of a positive and a negative point (if there are no negative or no positive points, $A$ returns the empty set).  These two pairs of points form a set $S_A$ that $A$ sends to $B$. $B$ now returns the classifier that optimally separates $D_B \cup S_A$, and if $S_A$ is empty then the interval classifier is as small as possible.

\begin{lemma}
\label{lem:intervals}
In $O(1)$ one-way communication we can find a $0$-error classifier $h \in \c{I}$.  
\end{lemma}
%\vspace{-.2in}
\begin{proof}
When $S_A$ is nonempty, this encodes two versions of Lemma \ref{lem:threshold}.  Assume without loss of generality that the positive points are contained in an interval with negative points lying outside the interval.  Then we can pick any positive point $p$ from either set $D_A$ or $D_B$ and consider the points greater than $p$ in the first instance of Lemma \ref{lem:threshold} and points less than $p$ in the second instance.  Invoking Lemma \ref{lem:threshold} proves this case.
When $S_A$ is empty, and a perfect classifier exists, then the minimal separating interval on $D_B$ will not violate any points in $S_A$, and will have no error.  
\end{proof}
%\vspace{-.15in}

%%%%%%%%%%%%%%%%%%%%%%%%%%%%%%%%%%%%%%%
%\subsection{Axis-Aligned Rectangles in $\mathbb{R}^d$}
%\label{ssec:aarects}	
\paragraph{Axis-aligned rectangles.}
We now consider finding a $0$-error classifier from the family $\c{R}^d$ of all axis-aligned rectangles in $\b{R}^d$.
%; we assume here such a classifier exists for data set $D$.  
An axis-aligned rectangle $R \in \c{R}^d$ can be defined by $d$-values in $\b{R}^d$, a minimum and maximum value along each coordinate axis. Given a data set $P$, the \emph{minimum axis-aligned rectangle} for $P$ is the smallest axis-aligned rectangle that contains all of $P$; that is, it has the smallest maximum coordinate possible along each coordinate axis and the largest minimum coordinate possible along each coordinate axis.  These $2d$ terms can be optimized independently as long as $P$ is non-empty.  

For a dataset $D_A$ we can define two minimum axis-align rectangles $R_A^+$ and $R_A^-$ defined on the positive and negative points, respectively.  If the positive or negative point set is empty, then each coordinate minimum and maximum is set to a special character $\emptyset$.  Two such rectangles can be defined for $D_B$ and $D = D_{A \cup B}$ in the same way. 

\begin{theorem}
\label{thm:aarects}
A one-way protocol where $A$ sends $R_A^+$ and $R_A^-$ to $B$ is sufficient to find a $0$-error classifier $h_{AB} \in \c{R}^d$ in the noiseless setting.  It requires $O(d)$ communication complexity.  
\end{theorem}
%\vspace{-.2in}
\begin{proof}
The key observation is that the minimum axis-aligned rectangle that contains $R_A^+$ and $R_B^+$ is precisely $R_{A \cup B}^+$ (and symmetrically for negative points).  Since the minimum and maximum for each coordinate axis is set independently, then we can optimize each using that value from $R_A^+$ and $R_B^+$.  Thus $B$ can compute this using points from $D_B$ and $R_A^+$.  

First, consider the case where positive points are inside the classifier and negative points are outside.  
Since there exist a $0$-error classifier $h^*$, then $R_{A \cup B}^+$ must be contained in that classifier, since no smaller classifier can contain all positive points.  
It follows by our assumption that $h^*$ and thus also $R_{A \cup B}^+$ contains no negative points, and can be returned as our $0$-error classifier $h_{AB}$.  
$B$ can determine if positive or negative points are inside by which of $R_{A \cup B}^+$ and $R_{A \cup B}^-$ is smaller.  If $R_A^+$ or $R_A^-$ is $\emptyset$, then $R_{A \cup B}^+ = R_B^+$ or $R_{A \cup B}^- = R_B^-$, respectively.  
\end{proof}

%-------------------------------------------------------------------------------------------------
\paragraph{Hyperplanes in $\mathbb{R}^2$.}
\label{ssec:lb-linsep}
The positive results from simpler geometric concepts do not extend to hyperplanes. 
We prove the following lower bound in Appendix \ref{app:lb-linsep}.
%Let $\c{L}^d$ denote the family of hyperplanes in $\mathbb{R}^d$.
\begin{theorem}
\label{thm:lb-linsep}
Using only one-way communication from $A$ to $B$, it requires $\Omega(1/\eps)$ communication to find an $\eps$-error linear classifier in $\b{R}^2$.  
\end{theorem}

Note that due to Theorem \ref{thm:id-net}, this is tight up to a $\log(1/\eps)$ factor for one-way communication.  

We can extend this lower bound to the $k$-node one-way model of computation where we assume each node $P_i$ can only send data to $P_{i+1}$.  In this case, we give node $A$'s input to $P_1$, and node $B$'s input to node $P_k$, and nodes $P_i$ for $i \in [2,k-1]$ have no data.  Then each node $P_i$ is forced to send the $\Omega(1/\eps)$ communication that $A$ wants to send to $B$ along the chain.  

\begin{theorem}
\label{thm:lb-linsep-k}
Using only one-way communication among $k$-players in a chaining model, it requires $\Omega(k/\eps)$ communication to find an $\eps$-error linear classifier in $\b{R}^2$.  
\end{theorem}

%%% Local Variables: 
%%% mode: latex
%%% TeX-master: "2011nips-main"
%%% End: 

%===============================================================================================
\section{\emph{Two-way} Two-Party Protocols for Linear Separators}
%\section{Linear Separators}
\label{sec:linsep}
%\vspace{-0.25cm}
%In this section, we present an upper bound on the communication complexity of learning a joint linear separator $h_{AB}$ that has $\eps$ error with respect to $h^*$. 
%We present an algorithm for the case of two dimensions.
%First, we present the algorithm for the case of two dimensions and subsequently generalize to higher dimensions.
In this section, we present a two-party algorithm that uses two-way communication to learn an
$\eps$-optimal combined classifier $h_{AB}$.
We also rigorously prove an $O(\log (1/\eps))$ bound on communication required.  
%that has $\eps$ error with respect to $h^*$. 
%We present an algorithm for the case of two dimensions.

%\subsection{2-d case}

%-----------------------------------------------------------------------------------------------
%\vspace{-0.15cm}
%\paragraph{Algorithm Overview.}
\subsection{Algorithm Overview}
%\vspace{-0.15cm}
Our algorithm proceeds in rounds. In each round both nodes send a constant number of points to the other.  The goal is to limit the number of rounds to $O(\log(1/\eps))$ resulting in a total communication complexity of $O(\log(1/\eps))$.  At the end of $O(\log(1/\eps))$ rounds of communication, the algorithm yields a combined classifier $h_{AB}$ that has $\eps$ error on $D$.

In order to bound the number of rounds, each node must maintain information about which points the
\emph{other node} might be classifying correctly or not at any stage of the algorithm. Specifically, suppose
node $A$ is sent a classifier $h_B$ from node $B$ (learned on $D_B$ and hence has zero error on $D_B$) and this classifier misclassifies some points in $D_A$. We denote these points as the  \emph{Set of Disagreement} (\sod) where \sod $ \subseteq D_A$. The remaining points in $D_A$ can be divided into the \emph{Set of Total Agreement} (\sota), which are the points on which classifiers from $A$ and $B$ will continue to agree on in the future, and the \emph{Set of Luck} (\sol), which are points on which the two nodes currently agree, but might disagree later on. The set of disagreement and the set of luck together form the \emph{Set of Uncertainty} $\sou = \sod \cup \sol$, representing all points that may or may not be classified incorrectly by $B$ in the future. 

Our goal will be to show that the \sou decreases in cardinality by a constant factor in each round. Achieving this will guarantee that at the end of $\log(1/\eps)$ rounds, the size of the \sou will be at most an $\eps$-fraction of the total input. Since $|\sou| \ge |\sod|$, we obtain the desired $\eps$-error classifier. 

The simplest strategy would be for each node to build a max-margin classifier on all points it has seen thus far, and send the support points for this classifier to the other node.  While this simple protocol might converge quickly in practice (we actually compare against it in Section \ref{sec:expts}, it is called \supp, and it often does), in principle this protocol may take a linear number of rounds to converge.  Thus, our algorithm will choose non-max-margin support vectors, but we will show that by sending these points we can achieve provable error and communication trade-off bounds. 

\begin{figure}
%\vspace{-0.15cm}
\centering 
\includegraphics[width=.8\columnwidth]{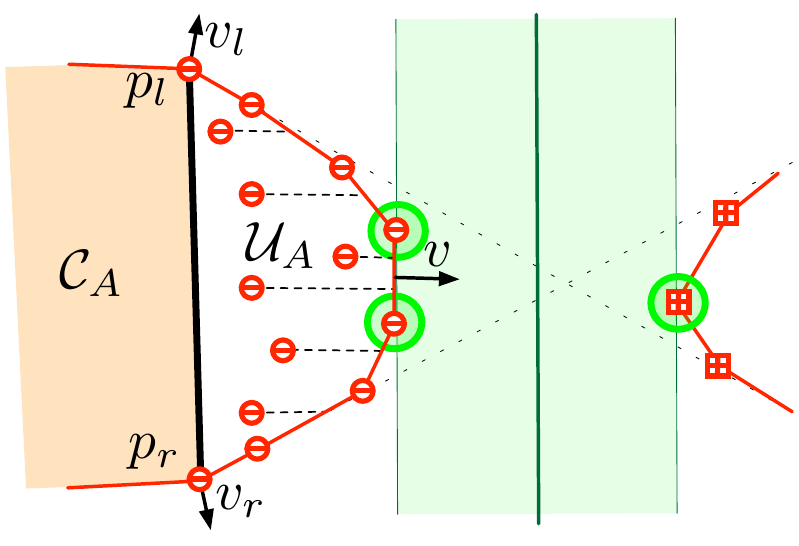}
\caption{\label{fig:init}$3$ support points chosen from $U_A$, and the family of $0$-error classifiers for $A$ parallel to $h_A$.  }
%\vspace{-0.25cm}
\end{figure}

%-----------------------------------------------------------------------------------------------
\subsection{The Algorithm}
\paragraph{Definitions and notation.}
Let $\c{P}^+_A$ and $\c{P}^-_A$ denote polytopes that contain positive and negative points in $D_A$, respectively.  
Let $\c{C}^+_A$ and $\c{C}^-_A$ denote the convex hulls formed by the positive and negative \sota in $D_A$ after the $i^{\text{th}}$ round, respectively.  
%We'll denote the convex hull of a set $X$ as $\c{C}(X)$.  
In general, when sets have a $+$ or $-$ superscript it will denote the restriction of that set to only positive or negative points, respectively.  Often to simplify messy, but usually straightforward, technical details we will drop the superscript and refer to either or both sets simultaneously. %; a more careful technical presentation is provided in Appendix \ref{app-sec:linsep}.  
%For instance, we will use $\c{P}_A$ and $C_{A,i}$, to refer to either one of the polytopes.  
We denote the \emph{region of uncertainty} $\c{U}_A$ as $\c{P}_A \setminus \c{C}_A$, and note $U_A = \c{U}_A \cap D_A$.  

In each round $A$ will send to $B$ a set $S_A \subset D_A$; these points imply a max-margin classifier $h_A$ on $S_A$ that has $0$ error on $D_A$; see Figure \ref{fig:init}. 
Then $B$ will either terminate with an $\eps$-error classifier $h_B$, or symmetrically return a set of points $S_B \subset D_B$.  
%will send in reply a set $S_{B,i} \subset D_B$.  
%Let $W_{B,i} = \bigcup_{j=1}^i S_{B,i}$ and $W_{A,i} = \bigcup_{j=1}^i S_{A,i}$ to refer to the union of all points set up through the $i$th round.  
%The set $S_{A,i}$ will include two parts: the support vectors of $D_A \cup W_{B,i-1}$ (except points already in $W_{A,i-1}$) and some additional points described below.  
%The support vectors in $S_{A,i}$ imply a max-margin classifier $h_{A,i}$ (and similar $h_{B,i}$).  
%
%
This process is summarized in Algorithm \ref{alg:ouralgo}.  

\begin{algorithm}[!htbp]
 \caption{\textsc{IterativeSupports}}
 \begin{algorithmic}
   \STATE \textbf{Input:} $D_A$ and $D_B$
   \STATE \textbf{Output:} $h_{AB}$ (classifier with $\eps$-error on $D_A \cup D_B$)
   \STATE $S_A := \csup(D_A)$; send $S_A$ to $B$;
   \WHILE{(1)} 
     \STATE --------- \textbf{B's move} ---------
		 \STATE compute error (err) using $h_A$ (from $S_A$) on $D_B$;
		 \STATE if(err $\leq$ $\eps|D_B|$) then exit;
		 %\STATE $D^+_B = D^+_B \cup S^+_A$; $D^-_B = D^-_B \cup S^-_A$;
     %\STATE $S_B := \csup(D^+_B) \cup \csup(D^-_B)$ 
		 \STATE $D_B = D_B \cup S_A$; $S_B := \csup(D_B)$; send $S_B$ to $A$;
     \STATE --------- \textbf{A's move} ---------
		 \STATE compute error (err) using $h_B$ (from $S_B$) on $D_A$;
		 \STATE if(err $\leq$ $\eps|D_A|$) then exit;
		 %\STATE $D^+_A = D^+_A \cup S^+_B$; $D^-_A = D^-_A \cup S^-_B$;
     %\STATE $S_A := \csup(D^+_A) \cup \csup(D^-_A)$ 
		 \STATE $D_A = D_A \cup S_B$; $S_A := \csup(D_A)$; send $S_A$ to $B$;
	 \ENDWHILE
 \end{algorithmic}
 \label{alg:ouralgo}
\end{algorithm}

Two aspects remain: determining if a player may exit the protocol with a $\eps$-error classifier (early termination), and computing the support points in the function \csup.

%-----------------------------------------------------------------------------------------------
%\vspace{-0.15cm}
\subsection{Early Termination}
%\vspace{-0.15cm}
Note that in Algorithm \ref{alg:ouralgo}, under certain \emph{early-termination} conditions, player $B$ may terminate the protocol and return a valid classifier, even if $h_A$ has more than $\eps$ error on $D_B$.  Any classifier that is parallel to $h_A$ and is shifted less than the margin of the max-margin classifier also has $0$ error on $D_A$.  Thus if any such classifier has at most $\eps$-error on $D_B$, player $B$ can terminate the algorithm and return that classifier.  

%\begin{wrapfigure}{r}{.5\textwidth}
\begin{figure}[!htbp]
\centering %\vspace{-.15in}
\includegraphics[width=\columnwidth]{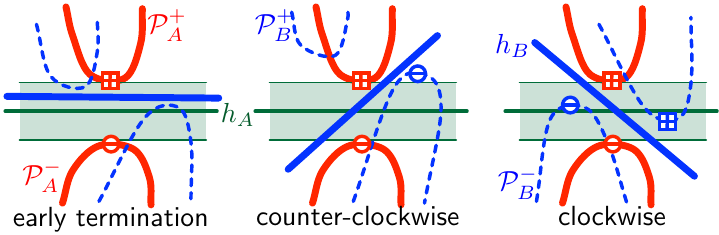}
%\vspace{-.25in}
\caption{Cases for either early termination, or for the direction of the normal to the linear separator being forced counter-clockwise or clockwise.\label{fig:early-term}}
\end{figure}
%\end{wrapfigure}

This early-termination observation is important because it allows $B$ to send to $A$ information regarding a $0$-error classifier, with respect to $h_A$, and the points $S_A$ that define it.  
If $B$ cannot terminate, then either some point in $D_B$ must be completely misclassified by all separators within the margin, or some negative point in $D_B$ and some positive point in $D_B$ must both be in the margin and cannot be separated; see Figure \ref{fig:early-term}.  Either scenario implies that any $\eps$-error classifier on $D_B$ must rotate in some direction (either clockwise or counter-clockwise) relative to $h_A$.  
This is important, because it informs $A$ that all points on $\partial \c{P}_A$ (the boundary of $\c{P}_A$) in the clockwise (resp. counter-clockwise) direction from $S_A$ will never be misclassified by $B$ if $h_B$ rotates in the counterclockwise (resp. clockwise) direction from $h_A$, increasing the SOTA, and decreasing the SOU.
This logic is formalized in Lemma \ref{app-lem:dir-choice}. %Appendix \ref{app-sec:linsep}.  

If the set $S_A$ always has half of $U_A$ on either side, then this process will terminate in at most $O(\log(1/\eps))$ rounds.  But, it may have no points on one side, and always be forced to rotate towards the other side.  
%\jeff{Perhaps claim we can actually construct such an example.}  
Thus, the set $S_A$ is chosen judiciously to ensure that $|U_A|$ decreases by at least half each round.  

%\vspace{-0.15cm}
\subsection{Choice of Support Points}
\label{ssec:choice}
%\vspace{-0.15cm}
What remains to describe is how $A$ chooses a set $S_A$, i.e. how to implement the subroutine \csup in Algorithm \ref{alg:ouralgo}.  If the set $S_A$ always has half of $U_A$ on either side, then this process will terminate in at most $O(\log(1/\eps))$ rounds, via the consequences of no early-termination.  But, if no points are on one side of $S_A$, and $B$'s response always forces $h_A$ to rotate towards the other side, then this cannot be assured.  
Thus, the set $S_A$ should be chosen judiciously to ensure that $|U_A|$ decreases by at least half each round.  

We present two methods to choose $S_A$.  This first does not have the half-on-either-side guarantee, but is a very simple heuristic, and which we show in Section \ref{sec:expts} often works quite well, even in higher dimensions.  The second, is only slightly more complicated and is designed precisely to have this half-on-either-side guarantee.  Both methods start by computing the region of uncertainty $\c{U}_A$ and the set of its points $D_A$ which lie in that region $U_A$.  

%The first is called \supo, and simply chooses the max-margin support points as $S_A$.  These points may include points previously sent over by $B$, otherwise this response would not change.  
The first is called \supo, and simply chooses the max-margin support points as $S_A$.  These points may
include points sent over in previous iterations from $B$ to $A$. %, otherwise this response would not change.  

The second is called \oalgo, and is summarized in Algorithm \ref{alg:compsup} (shown from $A$'s perspective).  It projects all of $U_A$ onto $\partial \c{P}_A$ (the boundary of $\c{P}_A$); this creates a weight for each edge of $\partial \c{P}_A$, defined by the number of points projected onto it.  Then \oalgo chooses the weighted median edge $E$.   %(See Appendix \ref{app-sec:linsep} for straightforward, technical details on how to process these edges jointly for $\partial \c{P}_A^-$ and $\partial \c{P}_A^+$).  
Finally, the orientation of $h_A$ is set parallel to edge $E$, and the corresponding support vectors are constructed.

\begin{algorithm}[!htbp]
 \caption{\csup implemented as \oalgo}
 \begin{algorithmic}[1]
   \STATE \textbf{Input:} $D = D_A \cup \{S_B\}$
   \STATE \textbf{Output:} $S_A$ (a set of support points)
   %\STATE $\c{P}^+$ := convex hull of $D^+$; $\c{P}^-$ := convex hull of $D^-$;
   %\STATE $\c{E}^+$ := convex edge of $\c{P}^+$; $\c{E}^-$ := convex edge of $\c{P}^-$;
   %\STATE $\c{U}^+$ := SOU of $\c{P}^+$; $\c{U}^-$ := SOU of $\c{P}^-$;
   \STATE project points in $U_A$ onto $\partial \c{P}_A$;
   \STATE $E$ := weighted median edge of $\partial \c{P}_A$; 
   \STATE $h_A$ := classifier on $D$ parallel to edge $E$;
   \STATE $S_A$ := support points of $h_A$;
	 %\STATE $sgn := sign(S)$;
	 %\STATE $S'$ := extreme point on $\partial \c{P}^{-sign(S)}$ along direction $-v$;
	 %\STATE $S$ := $S \cup S'$;
 \end{algorithmic}
 \label{alg:compsup}
\end{algorithm}

\section{Analysis of \textsc{IterativeSupports}}
In this section, we formally prove the number of rounds required by \textsc{IterativeSupports} to converge.

\subsection{The Basic Protocol}
To simplify the exposition of the protocol, we start with a special case, where player $A$ must, through interaction with $B$, teach $B$ parameters of classifier that has at most $\eps$ error on $D_A^-$, as well as some (but not all) negative examples in $D_B$. 
%\jeff{above doesn't read well - revisit}
This case captures the bulk of the technical development of the overall protocol. 
In Section~\ref{sssec:overall} we will then describe how to extend the protocol to
\begin{inparaenum}[(a)]
\item ensure at most $\eps$ error on both positive and negative examples in $D_A$, and
\item be symmetric: have at most $\eps$ error on $D_A \cup D_B$
\end{inparaenum}

We will describe the protocol from the point of view of player $A$. Each round of communication will start with $A$ computing a classifier from its current state, and sending support points for this classifier to $B$. $B$ then performs some computation, and either terminates returning an $\eps$-error classifier, or returns a single bit of information to $A$. $A$ updates its internal state, completing the round.  
%, or terminates, returning the current classifier as the desired answer. 

% \jeff{changed $v_l,v_u$ and $v_1, v_2$ and $p_1,p_2$ to a consistent $v_l,v_r$ and $p_l,p_r$.  I think its convenient to think of \emph{left} and \emph{right} even if we do not explicitly define them as such.  More intuitive than $1$, $2$.}

%\begin{wrapfigure}{r}{0.27\textwidth}
\begin{figure}
\centering %\vspace{-.1in}
\includegraphics[width=.8\columnwidth]{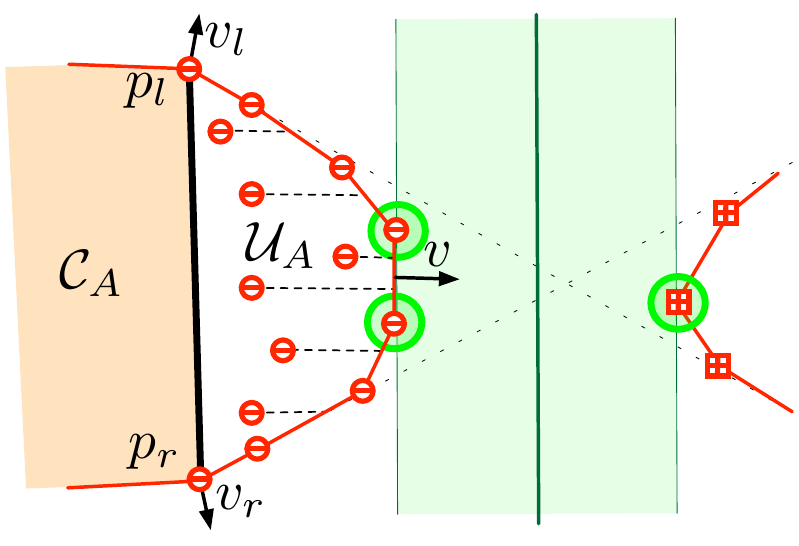}
\label{app-fig:init}
\end{figure}
%\end{wrapfigure}

\paragraph{Internal state.}
At any stage, $A$ maintains an interval of directions $(v_l, v_r) \subset \mathbb{S}^1$ where by convention, we go clockwise from $v_l$ to $v_r$. This interval represents $A$'s current bound on the possible directions normal to an $\eps$-optimal classifier based on all conversation with $B$ up to this point. $A$ also maintains $\c{C}_A$ (recall that $\c{C}_A$ is the convex hull of the \sota) as well as the set of points $U_A$ that form the \sou. By Lemma~\ref{app-lem:U-convex}, we know that $\c{P}_A = \c{C}_A \cup \c{U}_A$, and therefore there exist a pair of points $\{p_l, p_r\}$ on $\c{P}_A$ whose supporting line segment separates $\c{C}_A$ and $\c{U}_A$. $A$ maintains this pair as well; in fact, $v_l$ and $v_r$ represent outward normals to $\c{P}_A$ at $p_l$ and $p_r$.

\paragraph{(1) $A$'s move:}
$A$ projects all points in $U_A$ onto the boundary of $\c{P}_A$, denoted $\partial \c{P}_A$, (the projection is orthogonal to the edge through $\{p_l,p_r\}$).  Each edge in $\partial \c{P}_A$ is weighted by how many points are projected to it (with boundary points being assigned arbitrarily to one of the two incident edges).  We select the two points on the boundary of edge $e$ which is the weighted median, and place these points in a set $S$.  
The normal direction to $e$ is $v$.  And the extreme positive point in $D_A$ along direction $-v$ is also placed in $S$.  Now the classifier $h_A$ is the max-margin separator of $S$, has $0$ error on $D_A$, and is parallel to $e$.  
Then $A$ sends $(v_l,v_r,v,S)$ to $B$.

%These points now form an ordered sequence in between $p_l$ and $p_r$. 
%Pick the median  of this sequence (if there are two, pick either arbitrarily); call it $p$. Determine an outward normal $v$ to $\c{P}_A$ at $p$; note that $v$ lies between $v_l$ and $v_r$ on $\mathbb{S}^1$. Let $h_A$ be the max-margin classifier restricted to direction orthogonal to $v$; determine its (at most three) support points $S$ with respect to $P_A$. Note that $p$ must be one of the support points. $A$ now sends $(v_l, v_r, v, S)$ to $B$.

%\jeff{Note that if we first choose $v$, and then choose the points $S$ as described, it does not necessarily define a max-margin classifier with direction orthogonal to $v$.  However, it does define a max-margin classifier with direction normal to \emph{some} direction normal to $p$, and this is sufficient, since it still will invoke a decision about $p$ as the median point.   May need to change writing a bit - but I'll leave this for next pass, since it may make a slightly more obfuscated at this stage.}

\paragraph{(2) $B$'s move:}
$B$ receives $(v_l,v_r,v,S)$ from $A$. It then determines whether there exists a classifier $h_B$ with normal $v$ within the margin defined by $S$ that correctly classifies all but an $\eps$-fraction of points in $B$. If so, $B$ sends $(h_B, 0)$ to $A$ and terminates, returning $h_B$. 
Suppose that such a classifier does not exist. Then by Lemma \ref{app-lem:dir-choice}, any 0-error classifier for $D_B$ must have a normal either in the interval $(v_l, v)$ or $(v, v_r$). If the former, $B$ returns $(+1)$ to $A$, else it returns $(-1)$.

\paragraph{(3) $A$'s update:}
If $A$ receives $(h, 0)$ from $B$, the protocol has terminated, returning $h$. 
If $A$ receives $(+1)$, it then updates its interval of directions to be $(v_l, v)$ and sets the support pair separating $\c{C}_A$ and $\c{U}_A$ to $(p_l, p)$. Similarly, if it receives $(-1)$, it updates the interval of directions to $(v, v_r)$ and sets the support pair to $(p, p_r)$. In both cases, it adds $p$ to $C_A$, updating $\c{C}_A$ accordingly.

%\paragraph{Initialization.}
%There are exactly two distinct classifiers that have zero error on $D_A$ while also touching $\c{P}_A^+$ and $\c{P}_A^-$, they each touch (in a non degenerate setting) one point on each of these polygons. $A$ computes the two classifiers and initializes $v_l$ and $v_r$ to be the normals for these classifiers, selecting $p_l$ and $p_r$ accordingly as the points on $\c{P}_A^-$ touched by the two classifiers.

%%%%%%%%%%%%%%%%%%%%%%%%%%%%%%%%%%%%%%%%%%%%%
%\begin{wrapfigure}{r}{0.25\textwidth}
\begin{figure}
\begin{center} %\vspace{-.3in}
\includegraphics[width=.6\columnwidth]{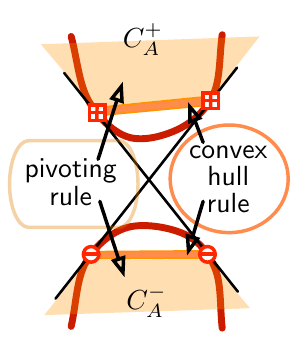}
\end{center} %\vspace{-.18in}
\label{app-fig:CH+pivot}
\end{figure}
%\end{wrapfigure}

\subsection{Structural Analysis}
In this section we provide structural results about $\c{C}_A$ and prove Lemma \ref{app-lem:U-convex} and Lemma \ref{app-lem:dir-choice}.  The first challenge is to reason about the set of total agreement -- what points can not be misclassified.  Then we can argue that $\sota = C_A \cap D_A$.  We use two technical tools, the convex hull and a pivoting argument.  
Let $W = \bigcup_i S_i$ be the union of all $S_i$ sent in round $i$ from $A$ to $B$.  

\begin{description}
\item[\textsf{Convex Hull:}] Let $\c{K}^- = \c{C}(W^-)$ be the convex hull of all the negative points sent by the protocol so far.  No negative points $p \in \c{P}_A^-$ can be misclassified if $p \in \c{K}^-$.  So $\c{K}^- \cap \c{P}_A^- \subset \c{C}_A^-$.  The same rule holds for positive points.  
\item[\textsf{Pivoting:}] Consider any point $q \in \c{P}_A^-$.  If any edge from $q$ to any point $p \in \c{K}^+$ intersects $\c{K}^-$, then $q$ cannot be misclassified -- otherwise a classifier which was correct on $p$ (and incorrect on $q$) would have to be incorrect on some negative point in $\c{K}^-$.  This identifies another part of $\c{P}_A^-$ as being in $\c{C}_A$, intuitively the region ``behind'' $\c{K}^-$.  
Note that the early-termination rotation argument, along with this pivoting rule, each round excludes from $U$ all points on one of two sides of the support points in $S$.  
\end{description}

We now have the tools to prove the two key structural lemmas needed for our protocol.

\begin{lemma}
\label{app-lem:dir-choice}
Consider when $B$ does not terminate.  
If $B$ returns $(+1)$, then $A$ can update its range to $(v_l,v)$.
If $B$ returns $(-1)$, then $A$ can update its range to $(v,v_r)$.
\end{lemma}
%\vspace{-.2in}
\begin{proof}
When $B$ can not produce an $\eps$-error separator parallel to $h_A$ and within the margin provided by $S$, that implies for any such classifier some points from $D_B$ must be misclassified.  
Furthermore, $B$ can present points $Y \subset D_B$ that along with $S$ violate any classifier orthogonal to $v$.  Let $y,s \in Y \cup S$ be a negative and positive point, respectively, one of which any classifier orthogonal to $v$ will misclassify.  
%Use $y$ to expand $\c{C}_A^+$ and $\c{K}^+$, and $s$ to expand $\c{C}_A^-$ and $\c{K}^-$ by the convex hull rule.  
Then any linear separator classifying $s$ and $y$ correctly must intersect the edge between $s$ and $y$, and thus must rotate from direction $v$ clockwise or counter-clockwise.  This excludes directions in either $(v_l,v)$ or $(v,v_r)$ and allows $B$ to return $(+1)$ or $(-1)$, accordingly.  
\end{proof}

\begin{lemma}
\label{app-lem:U-convex}
After $A$ has updated its state (step (3)), then $\c{U}_A$ is convex.
\end{lemma}
%\vspace{-.2in}
\begin{proof}
First consider the two negative points $\{p_l,p_r\}$.  Using the convex hull rule, the edge $e_{12}$ between them is in $\c{C}_A$.  
And because the points $\{p_l,p_r\}$ are defined as the extremal points for the range $(v_l,v_r)$ under the pivoting rule, everything ``behind'' them in $\c{P}_A$ is also in $\c{C}_A$.  
Thus, $\c{C}_A$ is partitioned from $\c{U}_A$ by the line passing through the edge $e_{12}$, implying that $\c{U}_A$ is convex.  
\end{proof}

%%%%%%%%%%%%%%%%%%%%%%%%%%%%%%%%%%%%%%%%%%%%%%%%%%%%%
\subsection{Extending The Basic Protocol}
\label{sssec:overall}
%\vspace{-.1in}

The simplified protocol above captures the spirit of $A$'s perspective of the algorithm on its negative points.  But to show it converges, we need to extend these techniques to also handle positive points and to make it symmetric from $B$'s perspective.  

\paragraph{Handling positive and negative instances simultaneously.} 
In each round of the basic protocol $U_A^-$ reduces in cardinality by at least half.  
We now describe how to modify the protocol so that the entire set $U_A = U^-_A \cup U^+_A$ is reduced in cardinality by half.  Recall that in step (1) of the basic protocol, $A$ projects all points in $U^-_A$ to the boundary of $\c{P}^-_A$ and determines a edge of the boundary that splits the set in half.  In addition now we project all points in $U^+_A$ to the boundary of $\c{P}^+_A$ as well. 
We can consider the normal direction of each edge in $\partial \c{P}^-_A \cap \c{U}^-$ or in $\partial \c{P}^+_A \cap \c{U}^+$ and map it to a point on $\b{S}^1$.  

%\begin{wrapfigure}{r}{0.6\textwidth}
\begin{figure}
\centering %\vspace{-.15in}
\includegraphics[width=.9525\textwidth]{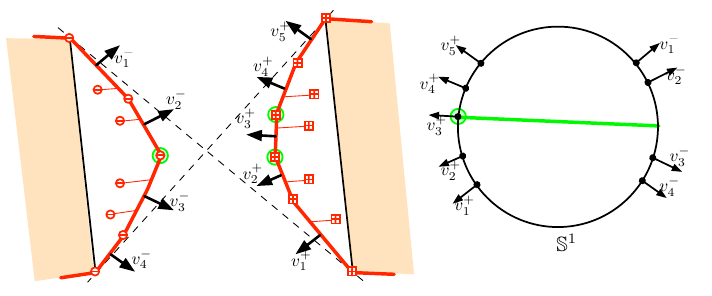}
%\vspace{-.25in}
%\label{fig:gauss-sort} 
\end{figure}
%\end{wrapfigure}

We can now scan both sets of normal directions on $\b{S}^1$ simultaneously by interleaving the order of directions from $\partial \c{P}^-_A \cap \c{U}^-$ with the antipodal directions from $\partial \c{P}^+_A \cap \c{U}^+$.  We again find the weighted median direction, corresponding to an edge, now among all negative and positive directions.  The set $S_A$ now consists of the two points defining the median edge as well as the point incident upon the two edges with normal directions on either side of the antipodal direction of the median edge.  

%\begin{wrapfigure}{r}{.5\textwidth}
%\begin{center} \vspace{-.2in}
%\includegraphics[width=.5\textwidth]{figures/median.pdf}
%\label{fig:median}  %%%% KEEP IN FULL VERSION FOR PROOF BELOW
%%\caption{\sffamily \small Illustration of the full iterative phase from the perspective of $A$'s negative points.  After projecting to $e$ (solid black edge), the median point $p$ is added to $S$ and marked.  It is shown how adding this median point $p$ to $S$ increases $C$.  Then either one of two sides of $p$ are added to $\c{C}_A$ based on $h_B$ in the next round.}
%\end{center}
%\end{wrapfigure}

As before, this splits the regions of uncertainty into two convex regions on each polytope. The bit returned by $B$ will guarantee that one region on each polytope will be eliminated, and by the above construction, this guarantees that we reduce the size of $U_A$ by a factor of two in each round. 

\begin{lemma}
\label{app-lem:split}
Over the course of a single round, the size of $U_A$ decreases by at least half.  
\end{lemma}
%%% KEEP PROOF IN FULL VERSION, BUT COMMENT IN SUBMITTED FOR SPACE
%\begin{proof}
%Any point in $U_A^-$ can be placed on one of three regions $L^-$, $R^-$, or $T^-$.  
%$T^- = \c{C}(\{p_l,p_r,p\})$.  These points are always in $C_A^-$ after the round by the convex hull rule.  
%$L^-$ (resp. $R^-$) is the region defined by the halfspace with boundary through $\{p_l,p\}$ (resp. $\{p,p_r\}$) intersected with $\c{P}^-_A$.  The same regions are defined symmetrically for positive points.  Each of $L^- \cup L^+$ or $R^- \cup R^+$ contain at most half of the points in $U_A$ by choosing the median edge.  By Lemma \ref{app-lem:dir-choice} one of the two ($L^-$ and $L^+$ or $R^-$ and $R^+$) must be $C^-_A$ and $C^+_A$ respectively after the round by the pivoting rule.  Only the points in the other of the two regions, $L$ or $R$, can remain in $U_A$, completing the proof.  
%\end{proof}

\paragraph{Reducing the \sou for both $A$ and $B$.}
The basic protocol and its extension described above only reduce the \sou for $A$. Since $B$ decides termination, it is possible that the error of the resulting classifier on $B$ never reduces sufficiently. While we could run protocols in parallel for $A$ and $B$, this could result in classifiers $h_A$ and $h_B$ that do not have $\eps$-error on the \emph{entire} data set $D_A \cup D_B$. 

The solution is for $B$ to send more information back to $A$. Consider step (2) of the basic protocol. $B$ receives a support set $S_A$ from $A$, as well as the set of directions $v_l, v, v_r$  and determines which of the intervals $(v_l, v)$ and $(v, v_r)$ the direction of a $0$-error classifier $h_B$ on $D_B$ must lie in. Now instead of merely sending back a bit, $B$ also sends back a support set $S_B$ corresponding to $h_B$, as well as its own directions $(v'_l, v'_r, v')$.  $A$ now uses the support set $S_B$ to update its own \sota and \sou, completing the round. 
%Note that since $A$ uses $S_B$ to update its classifier, we are guaranteed that any subsequent classifier $A$ sends to $B$ will never increase $B$'s \sou.  
%
Notice that now $B$'s transmission to $A$ in step (2) of the protocol is identical to $A$'s transmission that initiates step (2)!  
Thus all future separators proposed by $A$ or $B$ must correctly classify the same set of points in the full protocol transcript.  
%Thus, the basic protocol (with the above extension for positive and negative examples) can be generalized to the single step: Prune $\c{U}_A$ based on $S_B$ and pivoting rule, project all points in $U_A$ to $\c{P}_A$, and sort to find median edge which is used to define $S_A$.  

%-----------------------------------------------------------------------------------------------
\subsection{Complexity Analysis}

\begin{theorem}
\label{thm:rounds}
The $2$-player two-way protocol for linear separators always terminates in at most $O(\log (1/\eps))$ rounds, using at most $O(\log (1/\eps))$ communication.  
\end{theorem}

\begin{proof}
By Lemma~\ref{app-lem:split} we know that as each round shrinks the region of uncertainty $\sou$ by half of its current size for both $A$ and $B$.  And we keep doing this until $|U_A| \leq \eps |D_A|$ or $|U_B| \leq \eps |D_B|$, then the early-termination condition must be reached. This can be achieved in $O(\log (1/\eps))$ rounds.
\end{proof}

%\begin{figure*}[!htbp]
%\centering \vspace{-1.9in}
%%\includegraphics[width=.7\textwidth]{plots/dataset1.pdf} \hspace{.1in}
%\hspace{0.15in}
%\includegraphics[width=7cm]{plots/dataset1.pdf} \hspace{-0.95in}
%%\includegraphics[width=.7\textwidth]{plots/dataset2.pdf} \hspace{.1in}
%\includegraphics[width=7cm]{plots/dataset2.pdf} \hspace{-0.95in}
%\includegraphics[width=7cm]{plots/dataset3.pdf}
%\\
%\vspace{-.20in}
%\hspace{-0.20in}
%(a) \done  \hspace{0.8in}
%(b) \dtwo  \hspace{0.8in}
%(c) \dthr	 \hspace{0.1in}
%\vspace{-.1in}
%\caption{Datasets \done, \dtwo and \dthr. Red points represent $A$, and blue points represent $B$. The
%	positive and negative examples (for both $A$ and $B$) are denoted by `$\circ$'s and `$\square$'s, respectively.}
%\label{fig:data}
%\end{figure*}

%===========================================================================================================

\section{\emph{Multiparty}}
\label{sec:kparty}
%\section{$k$-Party Protocols for Linear Separators}
In the noiseless setting, extending from a two-party protocol to a $k$-party (where data is distributed to
$k$ disjoint nodes) can be achieved by allowing an additional factor $k$ or $k^2$ communication, depending on the hypothesis class.  

\subsection{One-way Protocols}

For $k$-players one-way protocols pre-determine an ordering among players $P_1 < P_2 < \ldots < P_k$, and all communication goes from $P_i$  to $P_{i+1}$ for $i \in [1,k-1]$.  In this section, we show that for $k$-players, $\eps$-error classifiers can be achieved even with this restricted communication pattern.  All discussed protocols can also be transformed into hierarchical one-way protocols that may have certain advantages in latency, or where all nodes just send information one-way to a predetermined coordinator node.  

\paragraph{Sampling results for $k$-players.}
In sampling-based protocols, along the chain of players, player $P_i$ maintains a random sample $R_i$ of size $O((\nu/\eps) \log (\nu/\eps))$ from $\bigcup_{j=1}^i D_i$ and the total size $m_i = \sum_{j=1}^i |D_i|$.  This can be easily achieved with reservoir sampling~\citep{vitter85reservoir}.  The final player $P_k$ computes and returns a $0$-error classifier on $R_{k-1} \cup D_k$.  
%If we use the one-way protocols based on sampling, then we can achieve $\eps$-error in $O(k(\nu/\eps) \log (\nu/\eps))$ communication.  
%Each node sends a sample of size $s_{\eps} = O((\nu/\eps) \log (\nu/\eps))$ to a coordinator, and the coordinator returns an $0$-error classifier on all of the data.  
\begin{theorem}
\label{thm:generickway}
Consider any family of hypothesis $(\b{R}^d, \c{A})$ that has VC-dimension $\nu$.  Then there exists a one-way $k$-player protocol using $O(k (\nu/\eps) \log (\nu/\eps))$ total communication that achieves $\eps$-error, with constant probability.  
\end{theorem}
\begin{proof}
The final set $R_{k-1}$ is an $\eps$-net, so any $0$-error classifier on $R_{k-1}$, is an $\eps$-error classifier on $\bigcup_{j=1}^{k-1} D_i$.  So since the total number of points misclassified is at most $\sum_{j=1}^{k-1} \eps |D_j| \leq \eps |D|$, this achieves the proper error bound.  The communication cost follows by definition of the protocol.  
\end{proof}

\paragraph{$0$-Error protocols for $k$-players.}
Any $0$-error one-way protocol extends directly from $2$-player to $k$-players.  This requires that each player can send exactly the subset of the family of classifiers that permit $0$ error to the next player in the sequence.  This chain of players only refines this subset, so by our noiseless assumption that there exists some $0$-error classifier, the final player can produce a classifier that has $0$-error on all data.  
\begin{theorem}
In the noiseless setting, any one-way two-player $0$-error protocol of communication complexity $C$ extended to a one-way $k$-player $0$-error protocol with $O(Ck)$ communication complexity.   
\label{thm:k-0err}
\end{theorem}

This implies that $k$-players can execute a one-way $0$-error protocol for axis-aligned rectangles with $O(dk)$ communication.  Classifiers from the families of thresholds and intervals follow as special case.

%We first consider the case of axis-aligned rectangles and show that in $\b{R}^d$, $k$-parties can find a $0$-error classifier using $O(kd)$ communication.  This includes intervals and thresholds as special cases.  One node is designated as the \emph{coordinator} and all other nodes send their minimal positive and minimal negative rectangles to the coordinator.  If the positive points are inside a rectangle that excludes the negative points, then the coordinator returns the minimal positive rectangle that contains all of the sent minimal positive rectangles and all of its positive points.  It proceeds symmetrically if all negative points are in a rectangle that excludes all positive points.  
%
%\begin{theorem}
%In the noiseless setting, $k$-parties can find a $0$-error classifier over axis-aligned rectangles in $\b{R}^d$ in $O(dk)$ communication.  
%\label{k-rect}
%\end{theorem}
%\begin{proof}
%The described protocol takes $4d$ to send information from $k-1$ nodes to the coordinator, achieving the communication bound.  
%We claim the $0$-error result since there must exist a $0$-error classifier in the noiseless setting, and it must contain all positive points.  Since each rectangle sent to the coordinator is minimal from each node, any rectangle smaller than the minimal on all data sent (and data at the coordinator node) would misclassify some positive point.  Since any $0$-error classifier cannot be any smaller along any axis, then this rectangle must be contained by the $0$-error classifier, and thus must itself have $0$ error.  
%\end{proof}

\subsection{Two-way Protocols}

When not restricted to one-way protocols, we assume all players take turns talking to each other in some preconceived or centrally organized fashion.  This fits within standard techniques of organizing communication among many nodes that prevents transmission interference.  

\paragraph{Linear separators in $\b{R}^2$ with $k$ players.}
Next we consider linear separators in $\b{R}^2$.  We proceed in a series of epochs.  In each epoch, each player takes one turn as coordinator.  On its turn as coordinator, player $P_i$ plays one round of the $2$-player protocol with each other player.  That is, it sends out its proposed support points, and each other player responds with either early termination or an alternative set of support points, including at least one that ``violates'' the family of linear separators proposed by the coordinator.  The protocol terminates if all non-coordinators agree to terminate early and their proposed family of linear separators all intersect.  Note that even if all other players may want to terminate early, they might not agree on a single linear separator along the proposed direction; but by replying with a modified set of support points, they will designate a range, and the manner in which these ranges fail to intersect will indicate to the coordinator a ``direction'' to turn.

\begin{theorem}
In the noiseless setting, $k$-parties can find an $\eps$-error classifier over halfspaces in $\b{R}^2$ in $O(k^2 \log (1/\eps))$ communication.                                                                                           
\label{thm:k-lin-sep}
\end{theorem}
\begin{proof}
Each epoch requires $O(k^2)$ communication; each of $k$ players uses a turn to communicate a constant number of bits with each of $k$ other players.  
We now just need to argue that the algorithm must terminate in at most $O(\log (1/\eps))$ epochs.  

We do so by showing that each player decreases its region of uncertainty by at least half for each turn it spends at coordinator, or it succeeds in finding a global separating half space and terminates.  If any non-coordinator does not terminate early, it rules out at least half of the coordinator's points in the region of uncertainty since by Lemma \ref{app-lem:split}, the coordinator's broadcasted support points represent the median of its uncertain points.  If all non-coordinators agree on the proposed direction, and return a range of offsets that intersect, then the coordinator terminates the algorithm and can declare victory, since the sum of all error must be at most $\sum_i \eps |D_i| \leq \eps |D|$ in that range.  

The difficult part is when all non-coordinators individually want to terminate early, but the range of acceptable offsets along the proposed normal direction of the linear separator do not globally intersect.  This corresponds to the right-most picture in Figure \ref{fig:early-term} where the direction is forced clockwise or counter-clockwise because a negative point from one non-coordinator is ``above'' the positive point from a separate non-coordinator.  The combination of these points thus allow the coordinator to prune half of its region of uncertainty just as if a single non-coordinator did not terminate early.  
\end{proof}

\section{Experiments}
\label{sec:expts}
%\vspace{-0.25cm}

In this section, we present results to empirically demonstrate the correctness and convergence of \itsupp.

%--------------------------------------------------------------------------------------------------------
%\vspace{-0.15cm}
\paragraph{Two-Party Results.}
%\vspace{-0.15cm}
For the two-party results, we empirically compare the following methods: 
%(a) \snode - a single node classifier learned on $D = D_A \cup D_B$, 
(a) \naiv - a naive approach that sends all points in $A$ to $B$ and then learns at $B$, %(similar to \snode), 
(b) \vote - a simple voting strategy that uses the majority voting rule to combine the predictions of $h_A$ and $h_B$ on $D = D_A \cup D_B$; ties are broken by choosing the label whose prediction has higher confidence,
(c) \rand - $A$ sends a random sample (an $\eps$-net $S_A$ of size $(d/\eps) \log(d/\eps)$) of $D_A$ to $B$ and $B$ learns on $D_B \cup S_A$, 
(d) \supo - %at each round, both $A$ and $B$ exchange support vectors of their current classifiers and the process repeats (
	\itsupp that selects informative points heuristically (ref. Section~\ref{sec:linsep}), 
	%%until the error in within the prescribed limit, 
and 
(e) \oalgo - \itsupp that selects informative points with convergence guarantees (ref. Section~\ref{sec:linsep}). 
SVM %(based on LibSVM~\citep{libsvm}) 
was used as the underlying classifier for all aforementioned approaches. In all cases, the errors are
reported on the dataset $D$ % = D_A \cup D_B$.
with an $\eps$ value of $0.05$ (where applicable).

The above methods have been evaluated on three synthetically generated 
datasets (\done, \dtwo, \dthr). For all datasets, both $A$ and $B$ contain $500$ data points each ($250$
positive and $250$ negative).
%\done and \dthr consists of $100$ examples ($50$
%positive and $50$ negative) whereas \dtwo consists of $70$ examples ($35$ positive and $35$ negative).
Figure~\ref{fig:two-party-data} pictorially depicts the data. 

\begin{figure*}[!htbp]
%\vspace{-0.45cm}
\centering
\subfigure[\done]{
\includegraphics[width=6.25cm]{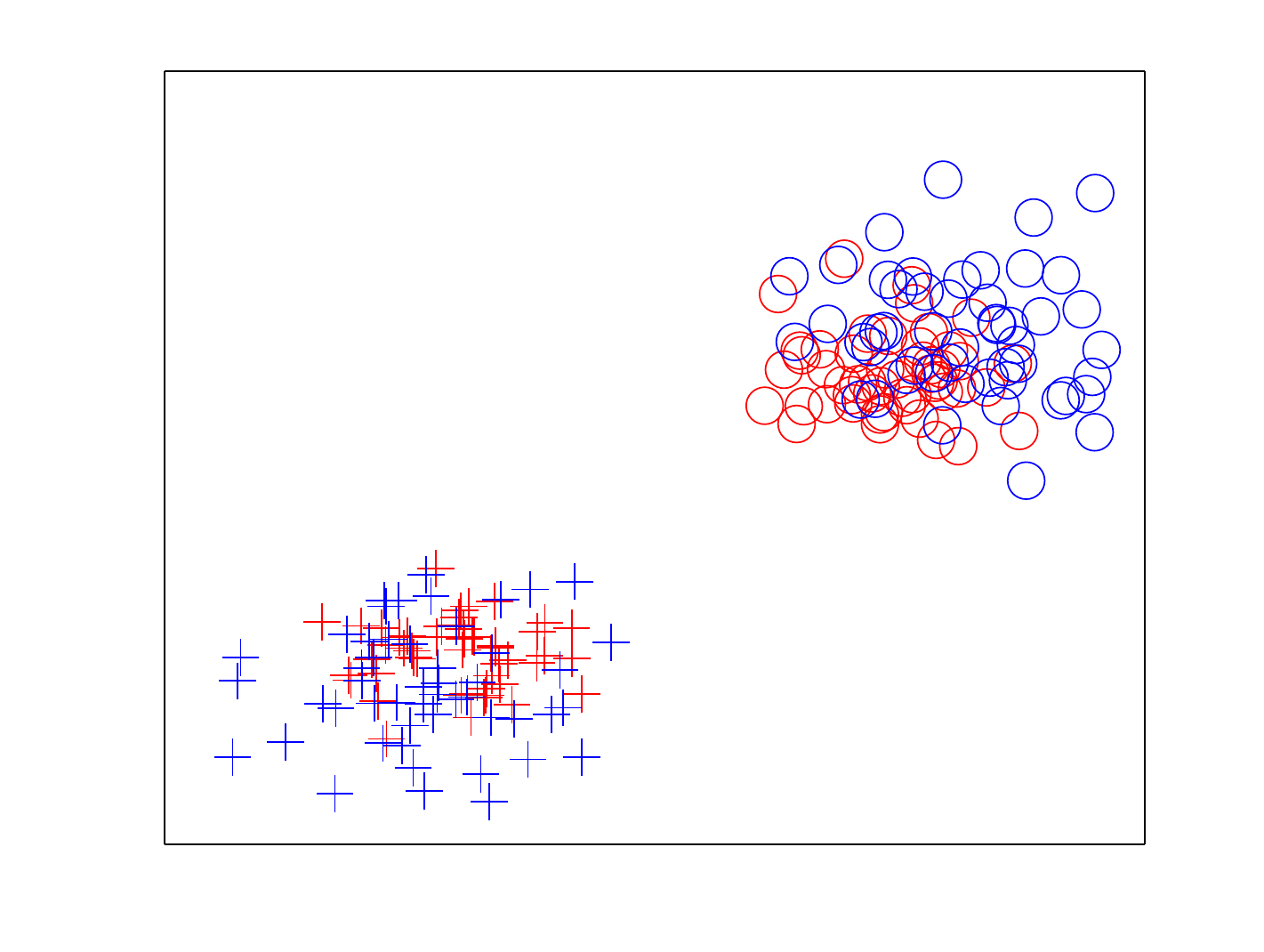}
\label{fig:data1-2party}
%\vspace{-.2in}
}
\subfigure[\dtwo]{
\includegraphics[width=6.25cm]{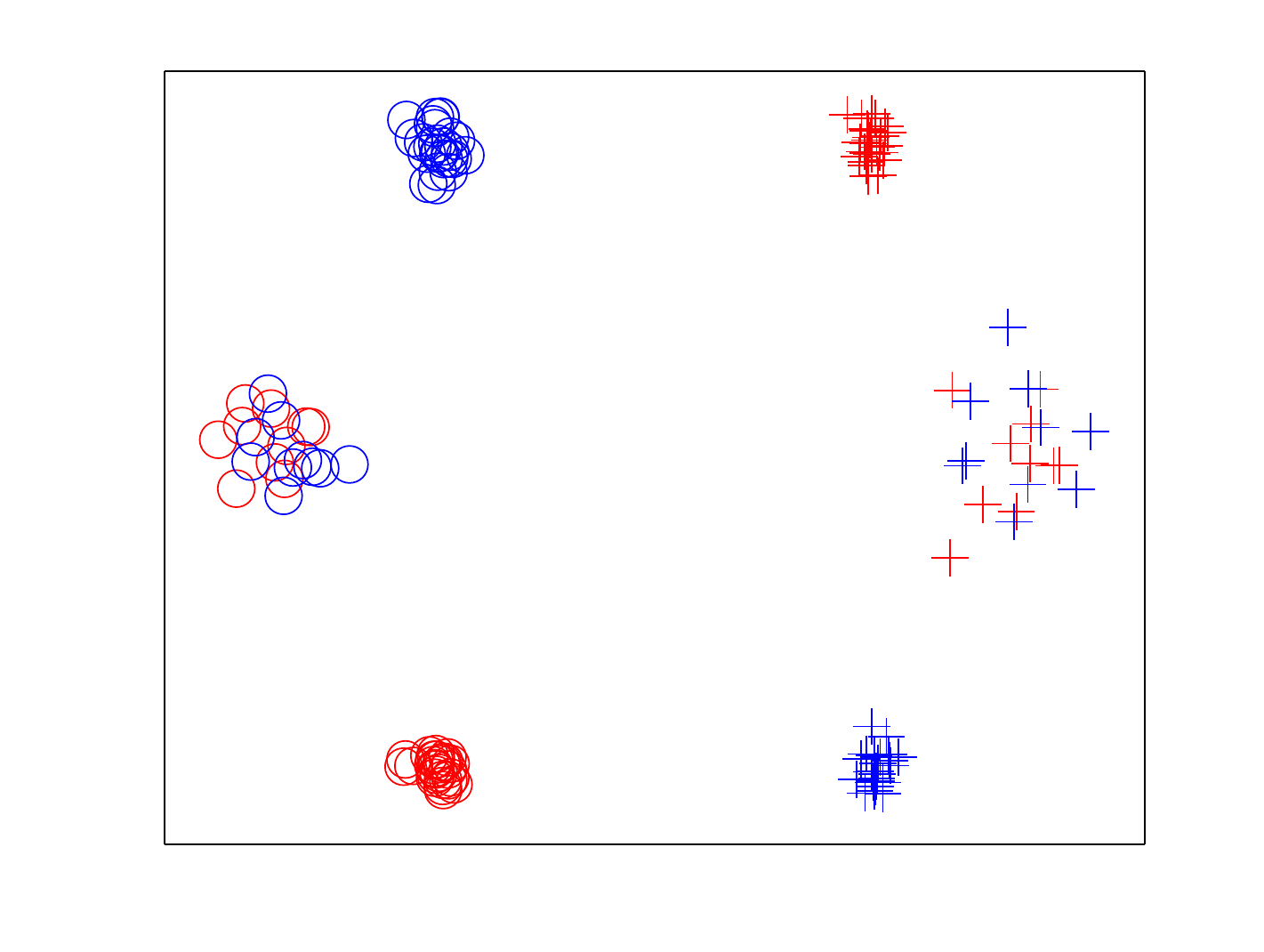}
\label{fig:data2-2party}
%\vspace{-.2in}
}
\subfigure[\dthr]{
\includegraphics[width=6.25cm]{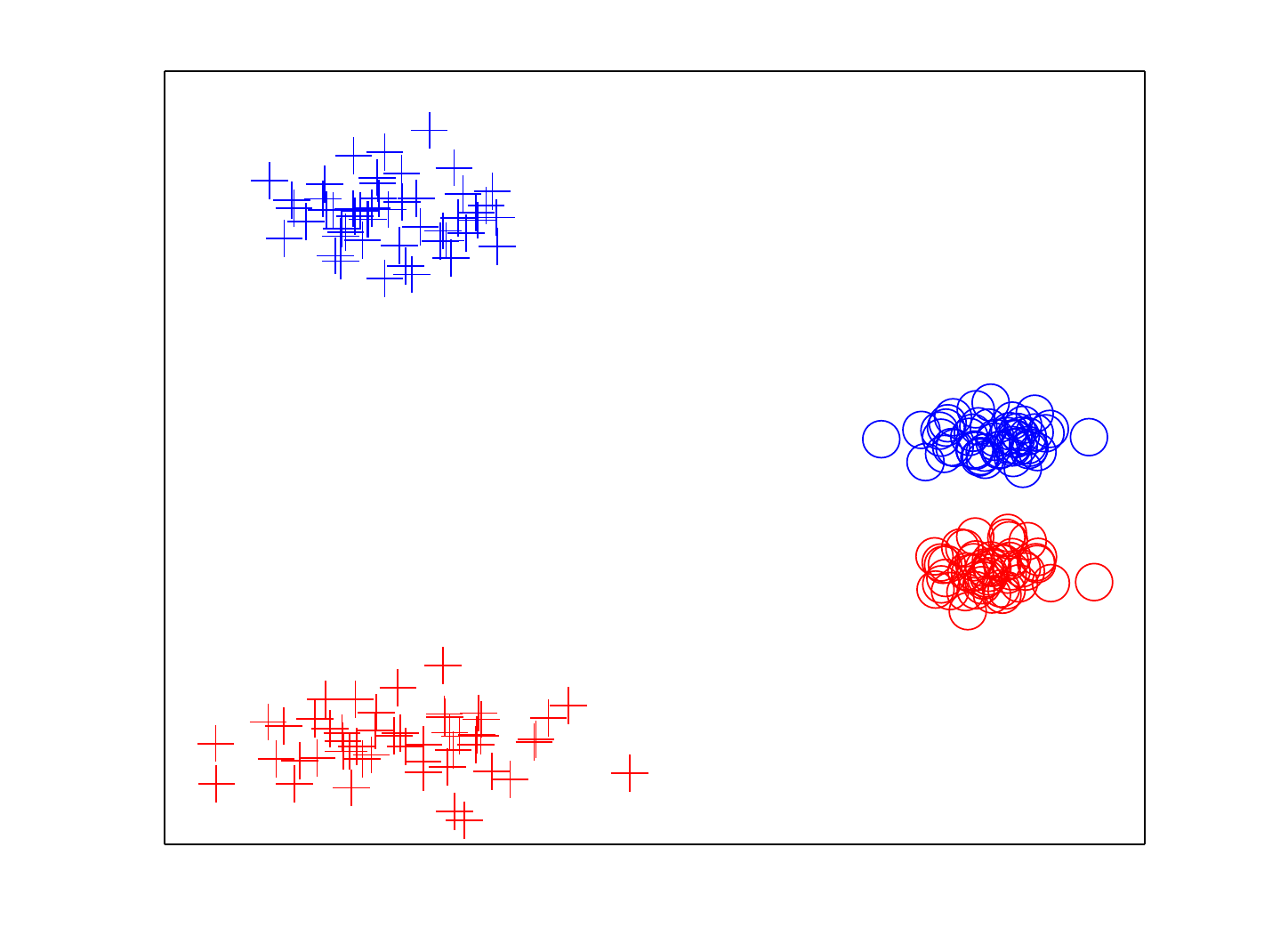}
\label{fig:data3-2party}
%\vspace{-.2in}
}
%\vspace{-.25in}
\caption{Red represents $A$ and blue represents $B$. Positive and
	negative examples (for all datasets) are denoted by `$+$'s and `$\circ$'s, respectively.}
\label{fig:two-party-data}
%\vspace{-0.25cm}
\end{figure*}

\begin{table}[!htbp]
	{\small
	\centering
	%\begin{tabular}{|p{1.5cm}||p{0.65cm}|p{0.5cm}|p{0.75cm}|p{0.5cm}|p{0.85cm}|p{0.5cm}|} 
	\begin{tabular}{|c||c|c|c|c|c|c|} 
		\hline
		{\bf Method} 	& \multicolumn{2}{|c|}{\bf \done} & \multicolumn{2}{|c|}{\bf \dtwo} & \multicolumn{2}{|c|}{\bf \dthr}\\
		\cline{2-7}
                  & {\bf Acc} 	& {\bf Cost}   & {\bf Acc} 	& {\bf Cost}  	& {\bf Acc} & {\bf Cost} \\
		\hline
%		\snode				& 100\%       & -						 & 100   \%   	& - 					&   100\%   	& -   					\\
		%\naiv         & 100\%       & 100          & 100   \%   	& 70  				&   100\%   	& 100 					\\
		%\vote         & 100\%       & 100          &  82.14\%   	& 70  				&    50\%   	& 100 					\\
		%\rand         & 100\%       & 60           &  99.29\%   	& 60  				& 99.57\%   	& 60  					\\
		%{\bf \supo}   & {\bf 100\%} & {\bf 4}      & {\bf 100\%} & {\bf 4}  		& {\bf 100\%} & {\bf 20}  		\\
		%{\bf \oalgo}  & {\bf 100\%} & {\bf 11}     & {\bf 100\%} & {\bf 13}			& {\bf 100\%} & {\bf 14}  		\\
		%%\naiv         & 100\%       & 500          &  100 \%   	& 500 				&   100\%   	& 500 					\\
		%%\vote         & 100\%       & 500          &  100 \%   	& 500  				&    80\%   	& 100 					\\
		%%\rand         & 100\%       & 24           &  100 \%   	& 24  				& 99.7\%   	& 24  					\\
		%%{\bf \supo}   & {\bf 100\%} & {\bf 4}      & {\bf 100\%} & {\bf 4}  		& {\bf 100\%} & {\bf 28}  		\\
		%%{\bf \oalgo}  & {\bf 100\%} & {\bf 6}     & {\bf 100\%} & {\bf 6}			& {\bf 100\%} & {\bf 11}  		\\
		\naiv         & 100\%       & 500          &  100\%   	& 500 				&   100\%   	& 500 					\\
		\vote         & 100\%       & 500          &  100\%   	& 500  				&    50\%   	& 500 					\\
		\rand         & 100\%       & 65           &  100\%   	& 65  				&   99.62\%   	& 65  					\\
		{\bf \supo}   & {\bf 100\%} & {\bf 4}      & {\bf 100\%} & {\bf 4}  		& {\bf 100\%} & {\bf 12}  		\\
		{\bf \oalgo}  & {\bf 100\%} & {\bf 6}     & {\bf 100\%} & {\bf 6}			& {\bf 100\%} & {\bf 10}  		\\
		\hline
	\end{tabular}
	%\vspace{-0.15cm}
	\caption{Accuracy (Acc) and communication cost (Cost) of different methods for \emph{two-dimensional} noiseless
		datasets.} %\done, \dtwo and \dthr. \rand, \supo and \oalgo use an $\eps=0.05$.}
	\label{tab:two-party-results-2d}
	}
%\vspace{-0.10cm}
\end{table}

Table~\ref{tab:two-party-results-2d} compares the accuracies and communication costs of the aforementioned
methods for the dataset in $2$-dimensions. For all datasets, \supo and \oalgo required the least amount of communication to learn an optimal classifier. For cases when it is easy to separate the positive from the negative samples (e.g. \done and \dtwo) \supo converges faster than \oalgo. However, 
%the theoretical analysis of \oalgo holds for any input dataset and 
\dthr show that there exists difficult datasets where \oalgo requires less communication than \supo. This reinforces our theoretical convergence claims for \oalgo that hold for \emph{any} input dataset. 
\dthr in Table~\ref{tab:two-party-results-2d} shows that there exists cases when both \vote and \rand perform worse than \oalgo and with a much higher communication overhead; for \dthr, \vote performs as bad as random guessing. 
%In fact, this difficult case can be generalized to $k$-players where $k/2$ players look like $A$ and $k/2$ look like $B$.  
Finally, neither \vote nor \supo provide any provable error guarantees.

%\subsection{high-d results}

\begin{table}[!htbp]
	{\small
	\centering
	%\begin{tabular}{|p{1.5cm}||p{0.65cm}|p{0.5cm}|p{0.75cm}|p{0.5cm}|p{0.85cm}|p{0.5cm}|} 
	\begin{tabular}{|c||c|c|c|c|c|c|} 
		\hline
		{\bf Method} 	& \multicolumn{2}{|c|}{\bf \done} & \multicolumn{2}{|c|}{\bf \dtwo} & \multicolumn{2}{|c|}{\bf \dthr}\\
		\cline{2-7}
                  & {\bf Acc} 	& {\bf Cost}   & {\bf Acc} 	& {\bf Cost}  	& {\bf Acc} & {\bf Cost} \\
		\hline
%		\snode				& 100\%       & -						 & 100   \%   	& - 					&   100\%   	& -   					\\
		\naiv         & 100\%       & 500          & 100\%   	  & 500  				&   100\%   	& 500 					\\
		\vote         & 100\%       & 500          & 100\%   	  & 500 				&   81.8\%   	& 500 					\\
		\rand         & 100\%       & 100           & 100\%   	& 100  				&   99.1\%   	& 100  					\\
		{\bf \supo}   & {\bf 100\%} & {\bf 4}      & {\bf 100\%} & {\bf 4}  		& {\bf 98.27\%} & {\bf 40}  		\\
%		{\bf \oalgo}  & {\bf 100\%} & {\bf 11}     & {\bf 100\%} & {\bf 13}			& {\bf 100\%} & {\bf 14}  		\\
		\hline
	\end{tabular}
	%\vspace{-0.15cm}
	\caption{Accuracy (Acc) and communication cost (Cost) of different methods for \emph{high-dimensional} noiseless
		datasets.}% \done, \dtwo and \dthr.} \rand and \supo use an $\eps=0.05$.}
	\label{tab:two-party-results-highd}
	}
%\vspace{-0.10cm}
\end{table}

Table~\ref{tab:two-party-results-highd} presents results for \done, \dtwo, \dthr extended to dimension $=10$. As can be seen, our
proposed heuristic \supo outperforms all other baselines in terms communication cost while having comparable
accuracies.

%\jeff{What dimension is \emph{high-dimensional}? and why does the size of random stay fixed?  The sample size should depend on $d$.  BTW, these numbers would look much more impressive with 1000 points instead of 100.}

%\vspace{-0.15cm}
%\subsection{$k$-Party Results}
%\label{ssec:kparty-res}
\paragraph{$k$-Party Results.}
%\vspace{-0.15cm}
The aforementioned methods have been appropriately modified for the multiparty scenario. For \naiv,
\vote and \rand, a node is fixed as the \emph{coordinator} and the remaining $(k-1)$ nodes send their
information to the coordinator node which aggregates all the received information. For \supo and \oalgo, in
each epoch, one of the $k$-players takes a turn to act as the \emph{coordinator} and updates its state by
receiving information from each of the remaining $(k-1)$ nodes. We experiment with a $k$ value of $4$ (i.e.,
four nodes $A, B, C, D$). As earlier, for all datasets each of $A$,$B$,$C$,$D$, contain $500$ examples ($250$
positive and $250$ negative). % whereas \dtwo consists of $70$ examples ($35$ positive and $35$ negative). 
The datasets are shown in Figure~\ref{fig:k-party-data}.

\begin{figure*}[!htbp]
%\vspace{-0.45cm}
\centering
\subfigure[\done]{
\includegraphics[width=6.25cm]{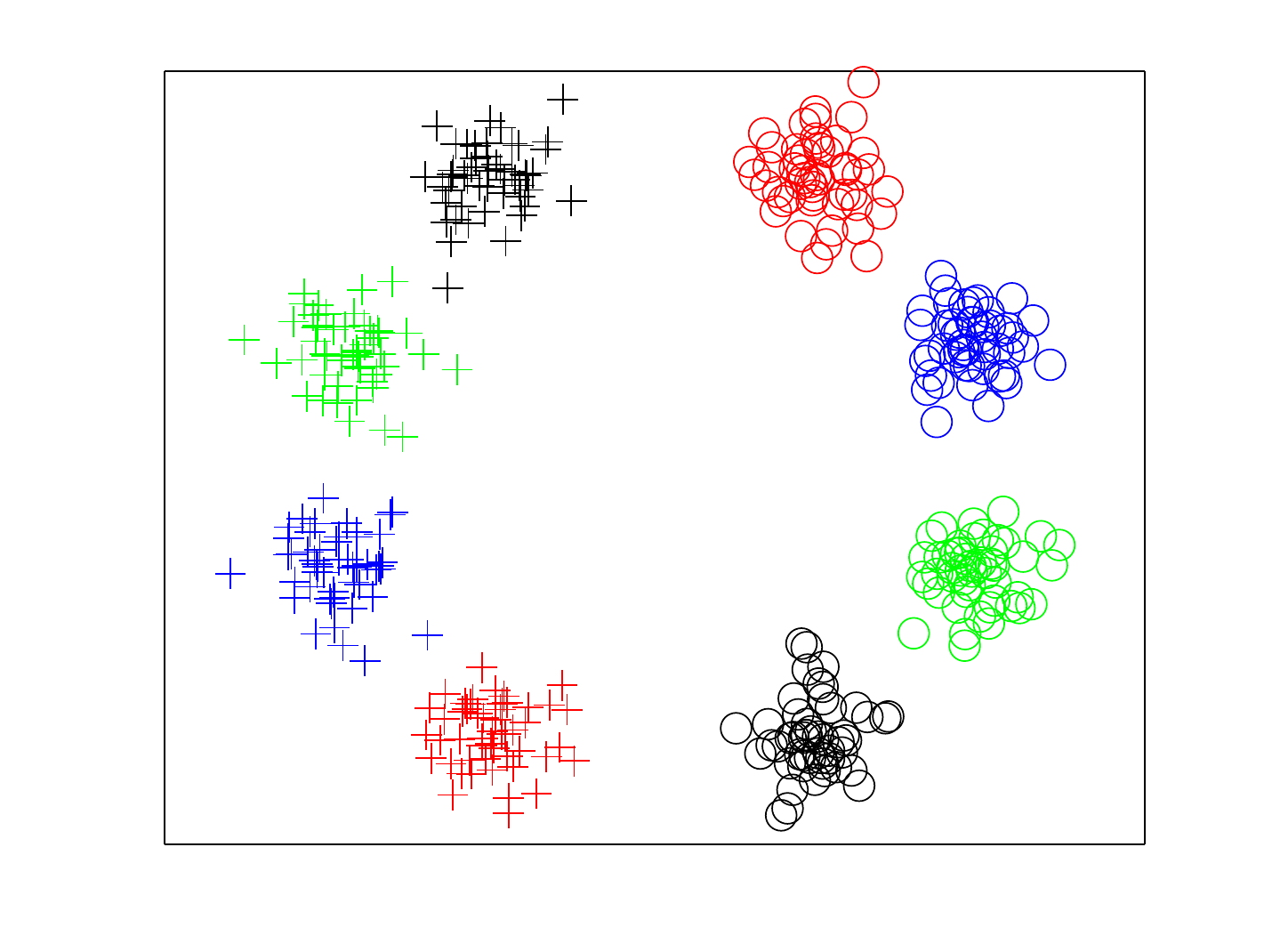}
\label{fig:data1-kparty}
%\vspace{-.2in}
}
\subfigure[\dtwo]{
\includegraphics[width=6.25cm]{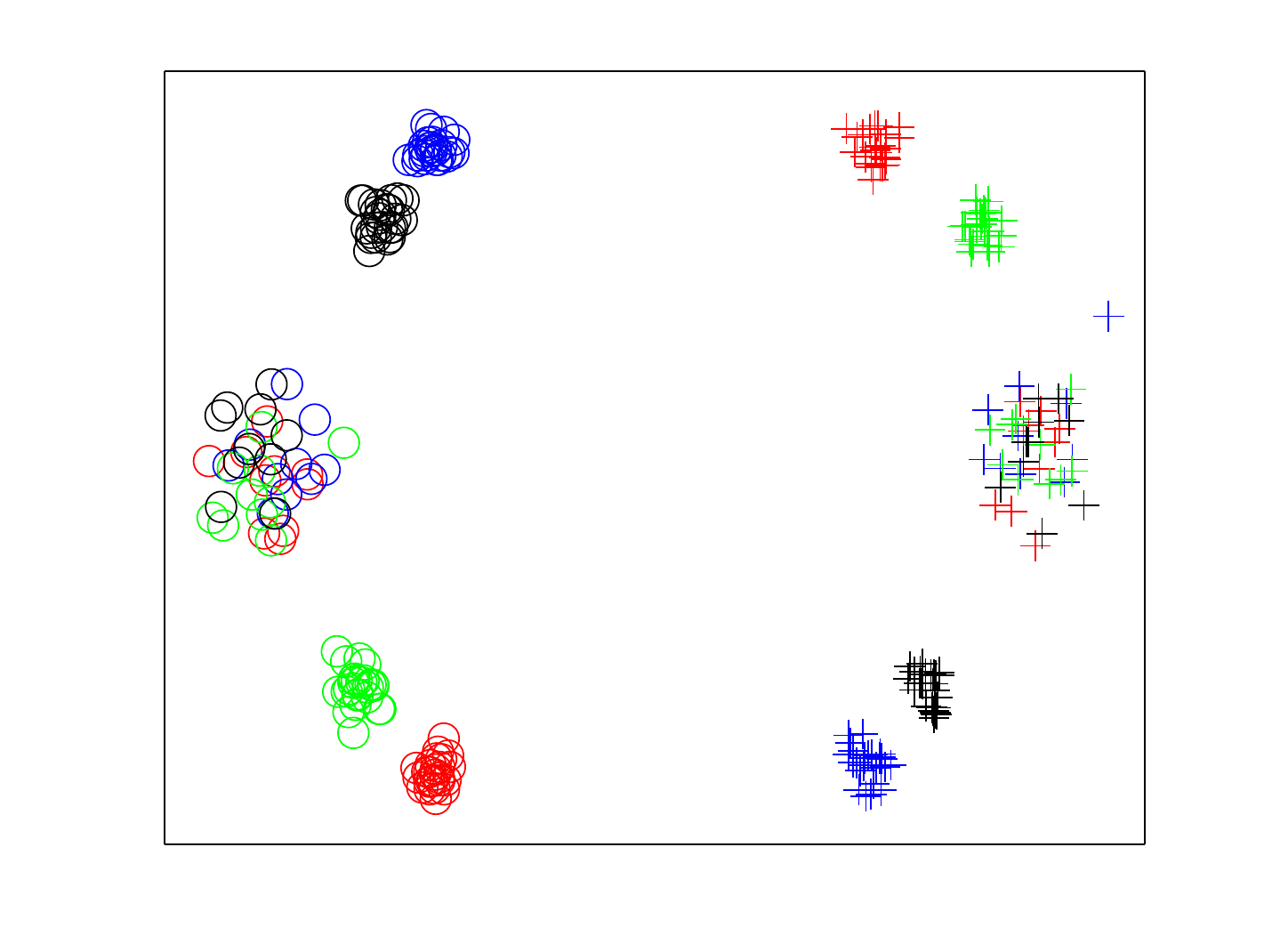}
\label{fig:data2-kparty}
%\vspace{-.2in}
}
\subfigure[\dthr]{
\includegraphics[width=6.25cm]{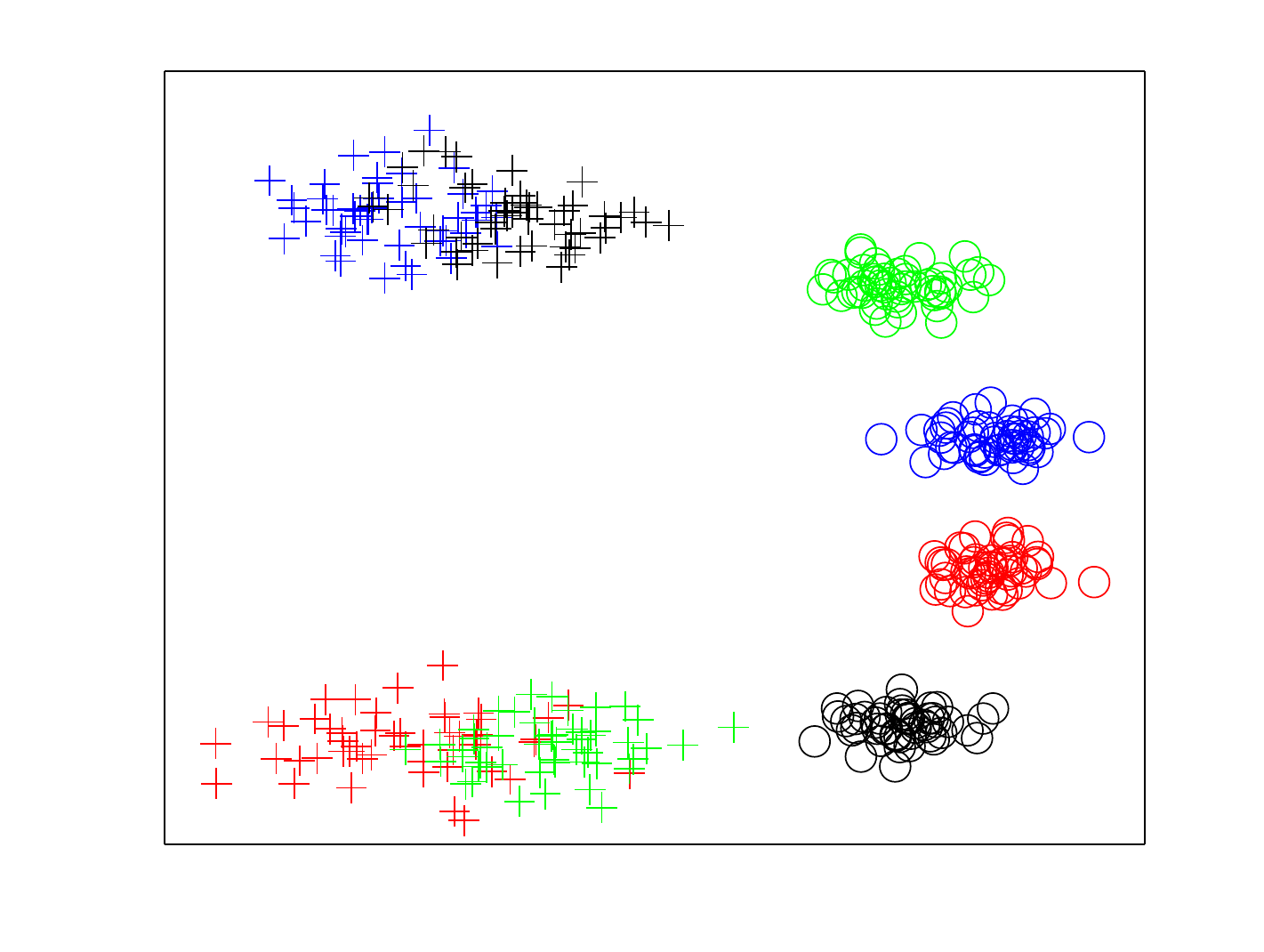}
\label{fig:data3-kparty}
%\vspace{-.2in}
}
%\vspace{-.1in}
\caption{Red represents $A$, blue represents $B$, green represents $C$ and black represents $D$. Positive and
	negative examples (for all datasets) are denoted by `$+$'s and `$\circ$'s, respectively.}
\label{fig:k-party-data}
%\vspace{-0.25cm}
\end{figure*}

\begin{table}[!htbp]
	{\small
	\centering
	%\begin{tabular}{|p{1.5cm}||p{0.85cm}|p{0.5cm}|p{0.65cm}|p{0.5cm}|p{0.85cm}|p{0.5cm}|} 
	\begin{tabular}{|c||c|c|c|c|c|c|} 
		\hline
		{\bf Method} 	& \multicolumn{2}{|c|}{\bf \done} & \multicolumn{2}{|c|}{\bf \dtwo} & \multicolumn{2}{|c|}{\bf \dthr}\\
		\cline{2-7}
                  & {\bf Acc} 	& {\bf Cost}   & {\bf Acc} 	& {\bf Cost}  	& {\bf Acc} & {\bf Cost} \\
		\hline
%		\snode				& 100\%       & -						 & 100   \%   	& - 					&   100\%   	  & -   					\\
		\naiv         & 100\%       & 1500         & 100\%      	& 1500				&   100\%   	  & 1500 					\\
		\vote         & 98.75\%     & 1500         & 100\%      	& 1500				&    50\%   	  & 1500 					\\
		\rand         & 100\%       & 195          & 100\%      	& 195 				&   99.76\%  	  & 195 					\\
		%{\bf \supo}   & {\bf 100\%} & {\bf 10}    & {\bf 100\%} & {\bf 2}  		& {\bf 100\%}   & {\bf 32}  		\\
		{\bf \supo}   & {\bf 97.61\%} & {\bf 14}   & {\bf 100\%} & {\bf 2}  		& {\bf 97.38\%} & {\bf 38}  		\\
		{\bf \oalgo}  & {\bf 99.0\%} & {\bf 36}    & {\bf 100\%} & {\bf 6}			& {\bf 98.75\%}  & {\bf 29}  		\\
		\hline
	\end{tabular}
	%\vspace{-0.15cm}
	\caption{Accuracy (Acc) and communication cost (Cost) of different methods for \emph{two-dimensional} noiseless
		datasets.} % \done, \dtwo and \dthr.}% \rand and \supo use an $\eps=0.05$.}
	\label{tab:k-party-results}
	}
%\vspace{-0.10cm}
\end{table}

As shown in Table~\ref{tab:k-party-results}, for the $k$-party case, \itsupp substantially outperforms the baselines on all
datasets. As earlier, for the difficult dataset \dthr, \oalgo incurs less communication cost as compared to
\supo. We observed that for \done and \dtwo, both \supo and \oalgo require the same number of iterations to converge.
However, the cost for \oalgo is higher due to its quadratic dependency on $k$. One of our future goals is to
get rid of an extra $k$ factor and reduce the dependency from quadratic to linear in $k$ (ref.
Section~\ref{ssec:exten}).

%and \dthr, \oalgo incurs less communication cost as compared to \supo. We note that as we move from $2$ to $k$ number of
%nodes, the performance of heuristic \supo degrades (when compared to \oalgo). In particular, for \done, we observed an oscillatory
%behavior of \supo where initially it had zero error on $D_A$ and full error on $D_D$. Thereafter it gradually
%increased its error on $D_A$ and reduced its error to $D_D$ and then
%again went back to zero error on $D_A$ and full error on $D_D$. The process was repeated a few times until \supo
%converged. This is indeed expected as \supo is only a heuristic without any
%shrinking-by-half guarantees on the region of uncertainty (ref. Section~\ref{ssec:choice}).

%Moreover, for all datasets \itsupp performs substantially better than the baselines. 

%===========================================================================================================
%\vspace{-0.45cm}
\section{Discussion}
\label{sec:disc}
%\vspace{-0.25cm}
This paper introduces the problem of learning classifiers across distributed data where the communication
between datasets is the bottleneck to be optimized.  This model focus on real-world communication bottlenecks
is increasingly prevalent for massive distributed datasets.  
In addition, this paper identifies several very general solutions within this framework and introduces new techniques which provide provable exponential improvement by harnessing two-way communication.  

%\vspace{-0.15cm}
\subsection{Comparison with Related Approaches}
\label{ssec:rel-work}
%\vspace{-0.15cm}
As mentioned earlier, techniques like classifier \emph{voting}~\citep{bauer99voting-emp} and 
\emph{mixing}~\citep{mcdonald10diststrucperc,mann09distmem} are often used in a distributed setting to obtain
global classifiers. Interestingly, we have shown that if the different classifiers are only allowed to train
on mutually exclusive data subsets then there exists specific examples (under the adversarial model) where
voting will \emph{always} yield sub-optimal results. We have presented such examples in Section~\ref{sec:expts}. 
Additionally, parameter mixing (or averaging~\citep{collins02discHMM}), which has been primarily
proposed for maximum entropy (MaxEnt) models~\citep{mann09distmem} and structured
perceptrons~\citep{mcdonald10diststrucperc,collins02discHMM}, have shown to admit convergence results but
lack any bounds on the communication. Indeed, parameter-mixing for structured perceptrons uses an iterative
strategy that performs a large amount of communication. %We also note that our results apply for specific hypothesis classes (e.g., axis-aligned rectangles, linear separators) and are not tied to any fixed classifier. 
%Finally, the parameter mixing based approaches cater exclusively to parameterized learners whereas our results have no such limitations.\suresh{there was something hal was concerned about here re: nonparametric learners. Not sure what the issue was.}

% This paper proposes active learning for massive datasets under a distributed learning framework. We present the problem setup, introduce a wide variety of possible variants and, as a first step, solve a core subset of the aforementioned scenarios. 
% The technical highlight of our paper is a two-way communication protocol that provides
% an exponential improvement in communication cost, as compared to its one-way counterpart. 
% More importantly, we show that the two-way result can be reused to design communication protocols for the multi-party case \avi{Check!!}. 
% The key conceptual highlight of our work is the idea of using active learning to reduce other forms of model cost (in this case communication complexity) which, we believe, is of independent interest to the research community.  
% Thus, this paper lays the ground work for an important new family of problems to deal with learning on massive data.   The work presented here (both technical and conceptual) opens numerous interesting future research directions extending our proposed model and building on the core set of technical results presented in this paper.

The body of literature that lies closest to our proposed model relates to prior work on label compression
bounds~\citep{floyd95samplecompression, helmbold95onweak}. 
In the label compression model, both $A$ and $B$ have the same data but only $A$ knows the labels. The goal
is to efficiently communicate labels from $A$ to $B$.
Whereas in our model, each player ($A$ and $B$) have ``disjoint labeled'' datasets and the goal is to
efficiently communicate so as to learn a combined final $\eps$-optimal classifier on $D_A \cup D_B$.
Indeed some of our \emph{one-way} results derive bounds similar to the cited work, as they all build on the theory
of $\eps$-nets. In particular, there exists a label compression method~\citep{helmbold95onweak} based on
boosting, which gives $O(\log{1/\eps})$ size set for \emph{any concept} that can be represented as a majority vote
over a fixed number of concepts. However, in our model, we show that for certain concept classes (with
\emph{one-way} communication) we need a linear amount of communication (ref. Theorem~\ref{thm:lb-linsep}).
Furthermore, we demonstrate that using a \emph{two-way} communication model can provide an exponential
improvement (ref. Theorem~\ref{thm:rounds}) in communication cost.

%\vspace{-0.15cm}
\subsection{Future Extentions}
\label{ssec:exten}
%\vspace{-0.15cm}
Although we have provided many core techniques for designing protocols for minimizing communication in learning classifiers on distributed data, still many intriguing extensions remain.  Thus we conclude by outlining three important directions to extend this work and provide outlines of how one might proceed.  

%There are many future directions to pursue such as extending the linear separator case to higher dimensions, extending these results to the noisy setting, and improving the bounds for $k$-player problems. 

%\vspace{-0.25cm}
\paragraph{Higher dimensions.}
We provide several results for high-dimensions: for axis-aligned rectangles, for bounded VC-dimension families of classifiers, and a heuristic for linear separators.  But for the most common high-dimensional setting--SVMs computing linear separators on data lifted to a high-dimensional feature space--our results either have polynomial dependence on $1/\eps$ or have no guarantees.  For this setting, it would be ideal to extend out \oalgo routine which requires $O(\log 1/\eps)$ communication in $\b{R}^2$ to work in $\b{R}^d$.  The key insight required is extending our choice of a \emph{median point} to higher dimensions.  Unfortunately, the natural geometric generalization of a \emph{centerpoint} does not provide the desired properties, but we are hopeful that a clever analysis of a constant size \emph{net} or \emph{cutting} of the space of linear separators will provide the desired bounds.

%\vspace{-0.25cm}
\paragraph{Noisy setting.}
Most of the results presented in this paper generalize to noisy data. In fact, Theorem~\ref{thm:id-net} and %\& %and 
Theorem
~\ref{thm:generic1way} have straight-forward extensions to the noisy case by an $\eps$-sample argument~\citep{sariel11cg}.  This would increase the communication from $O(1/\eps)$ to about $O(1/\eps^2)$.  
It would of course be better to use communication only logarithmic in $1/\eps$.  
We suggest modifying \itsupp to work with noisy data with the following heuristic, and defer any formal analysis.  
In implementing \csup (based either on \supp or \oalgo) we suggest sending over support points of linear
separators that allow for classifiers with exactly $\eps$-error.  That is, players never propose classifiers
with $0$-error, even if one exists; or at least they provide margins on classifiers allowing $\eps$-error.  This would seem to describe the proper family of classifiers tolerating $\eps$-error of which we seek to find an example.  

%\vspace{-0.25cm}
\paragraph{Efficient two-way $k$-party protocols.}
All simple one-way protocols we present generalize naturally and efficiently to $k$-players; that is, with only a factor $k$ increase in communication.  In fact, a distributed random sample of size $t = O((\nu/\eps) \log (\nu/\eps))$ can be drawn with only $O(t + k)$ communication~\citep{HYLC11}, so under a different two-way \emph{coordinator model} some results for the one-way \emph{chain model} we study could immediately be improved.  
However, again it would be preferable to achieve protocols for linear separators with communication linear in $k$ and logarithmic in $1/\eps$; our protocols are quadratic in $k$.  In particular, our protocol seems slightly wasteful in that each player is essentially analyzing its improvements independently of  improvements obtained by the other players.  To improve the quadratic to linear dependence on $k$, we would need to coordinate this analysis (and potentially the protocol) to show that the joint space of linear separators must decrease by a constant factor for each player's turn as coordinator, at least in expectation.

%===========================================================================================================
% Acknowledgements should go at the end, before appendices and references
\section{Acknowledgement}
This work was sponsored in part by the NSF grants CCF-0953066 and CCF-0841185 and in part by the DARPA CSSG grant N11AP20022. This work was also partially supported by the sub-award CIF-A-32 to the University of Utah under NSF award 1019343 to CRA. All the authors gratefully acknowledge the support of the grants. Any opinions, findings, and conclusion or recommendation expressed in this material are those of the author(s) and do not necessarily reflect the view of the funding agencies or the U.S. government.

% Manual newpage inserted to improve layout of sample file - not
% needed in general before appendices/bibliography.
\newpage

%===========================================================================================================
%===========================================================================================================
\appendix

%===========================================================================================================
\section*{Appendix}

In all cases below we reduce to the \emph{indexing problem}~\cite{kushilevitz97cc}: Let $A$ have $n$ bits either $0$ or $1$, and $B$ has an index $i \in [n]$.  It requires $\Omega(n)$ one-way communication from $A$ to $B$ for $B$ to determine if $A$'s $i$th bit is $0$ or $1$, even allowing a $1/3$ probability of failure under randomized algorithms. 

%===========================================================================================================
\section{Lower Bounds for One-Way Linear Separators}
\label{app:lb-linsep}
\begin{theorem}
Using only one-way communication from $A$ to $B$, it requires $\Omega(1/\eps)$ communication to find an $\eps$-error linear classifier in $\b{R}^2$.  
\end{theorem}
\begin{proof}
We consider linear separators in $\mathbb{R}^2$ and suppose that points in $D_A$ and $D_B$ are distributed (almost) on the perimeter of a circle.  This generalizes to higher dimensional settings by restricting points to lie on a $2$-dimensional linear subspace.
Figure~\ref{fig:lb-linsep} shows a typical example where $D_A$ has exactly $1/\eps$ negative points around a circle (each lies almost on the circle).  These points form
$1/2\eps$ pairs of points, each close enough to each other and to the circle that they only effect points within the pair. Each pair can have two configurations:
%\vspace{-0.35cm}
\begin{packeditemize}
	\item {\bf Case 1:} left point just inside circle and right point just outside circle (red disks)
  \item {\bf Case 2:} right point just inside circle and left point just outside circle (red boxes)
\end{packeditemize}

\begin{figure*}[!htbp]
	%\vspace{-0.25cm}
	\centering
	\includegraphics[width=0.7\textwidth]{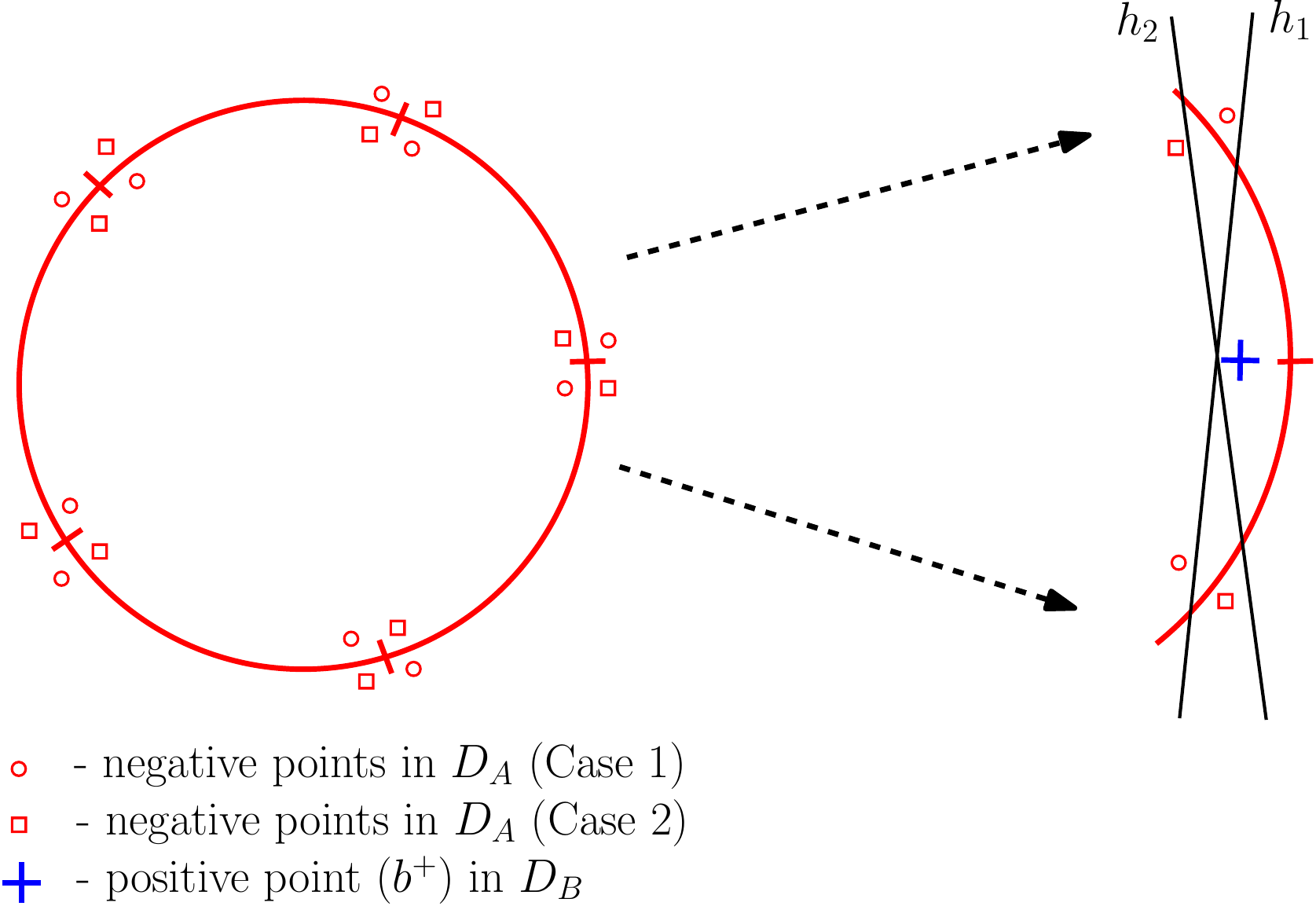}
	\caption{An example to prove the lower bound results for one-way communication with
		linear separators. The figure on the left shows the distribution of the negative points in $D_A$ for case
		1 and case 2. The right figure zooms only a small arc of the circle and shows what happens when $B$ decides
		the final classifier based on its single positive point $b^+$ and all negative points from $D_A$.}
	\label{fig:lb-linsep}
\end{figure*}

$D_B$ has only one positive point $b^+$ (blue plus) that interacts with exactly one pair of points from $D_A$, but
$B$ does not know which pair to interact with ahead of time. The positive point $b^+$ is placed close to the arc of the
circle with equal arc length to the negative points from $D_A$ on its either sides such that it is just
inside the circle. % (just off the edge between the pair of points).

\begin{claim}
Let $Z_j$ be a pair of points in $D_A$ and let $x_j$ be the position of $b^+$ (with respect to $Z_j$), as
shown in Figure~\ref{fig:lb-linsep}. If $D_B$ has a point $b^+$ at $x_j$, then $A$ needs to send at least one
bit of information about $Z_j$ to $B$, for $B$ to learn the perfect classifier.
\end{claim}
\begin{proof}
Suppose $A$ sends no information to $B$.
In order to learn an optimal classifier, $B$ makes the classifier tangent to the circle but offset to just
include its point $b^+$. However, the point $b^+$ is so positioned that it always forces the classifier
learned by $B$ to misclassify either the left negative point in $D_A$ (classifier $h_1$ in
Figure~\ref{fig:lb-linsep}, if case 1) or the right negative point in $D_A$ (classifier $h_2$ in 
Figure~\ref{fig:lb-linsep}, if case 2), whichever point is just outside the circle. 
$B$ can guess case 1 and angle the classifier to the left point, or guess case 2 and angle the classifier to the right point. 
But in either case, without any information from $A$, it will be wrong  half the time. 
%And one point misclassified by $B$ incurs an $\eps$-error overall (as we have a total of $1/\eps$ number of points in $D_A$). 
However, if $A$ sends a single bit of information denoting whether some negative point pair belongs to case 1 or case 2 then $B$ can use this information to learn a perfect separator with zero error.
\end{proof}

If we increase the number of points to $n$, by putting $\eps n$ identical points at each point in the
construction then such misclassified points cause an $\eps$ error. % and in total we have $\eps|D_A|/2$ error.

We also note that, each pair of points are independent of the others and a classifier learned for any one
negative point pair (in $D_A$) works for other negative point pairs.

In the above case, $A$ has $1/2\eps$ point pairs that are all negative.  Each pair is far enough away from all other pairs so as not to affect each other.  
$B$ has $1$ positive point placed as shown on the right of Figure \ref{fig:lb-linsep} for some point pair,
	not known to $A$.
To reduce this problem to indexing, we let each of $A$'s point pairs to correspond to one bit which is $0$ (if case 1) or $1$ (if case 2). 
And $B$ needs to determine if the $i$th bit (corresponding to the negative point pair in $D_A$ which $b^+$ needs to deal with) is $0$ or $1$. 
This requires $\Omega(1/\eps)$ one-way communication from $A$ to $B$, proving the lemma.
\end{proof}

%===========================================================================================================
\section{Lower Bounds for One-Way Noise Detection}
\label{app:lb-noisedetect}
Although in the noiseless non-agnostic setting we can guarantee to find optimal separators with one-way
communication, under the assumption they exist, we cannot detect definitively if they do exist.  For
intervals, the difficult case is when $A$ has only negative points, and for axis-aligned rectangles the
difficult case is more general.  

\begin{lemma}
It requires $\Omega(|D_A|)$ one-way communication from $A$ to $B$ to determine if there exists a perfect classifier $h \in \c{I}$.  
\end{lemma}
\begin{proof}
Consider the case where $A$ has $n/2$ points and they are all negative.  All of its points have values in $[2n]$ and are even.  $B$ has $2$ positive points and $n/2-2$ negative, its points have values in $[2n+1]$ and are all odd.  Its two positive points are consecutive odd points, say $2i-1$ and $2i+1$.  If $A$ has a point at index $2i$, then there is no perfect classifier, if it does not, then there is.  

This is precisely the indexing problem with $A$'s points corresponding to a $1$ if they exist for index $2i$ and to a $0$ if they do not, and for $B$'s index $i$ corresponding to the value $i$ for which it has positive points at $2i-1$ and $2i+1$.  Thus, it requires $\Omega(|D_A|)$ one-way communication, proving the lemma.  
\end{proof}

\begin{lemma}
It requires $\Omega(|D_A|)$ one-way communication from $A$ to $B$ to determine if there exists a perfect classifier $h \in \c{R}_2$, even if $A$ and $B$ have positive and negative points.  
\end{lemma}
\begin{proof}
Let $A$ and $B$ both have a positive point at $(2n,0)$ and a negative point at $(0,2n)$.  
$A$ also has a set of $n/2-2$ negative points at locations $(2i,2i)$ for some distinct values of $i \in [n]$.  $B$ has a (variable) positive point at some location $(2i-1,2i+1)$ for $i \in [n]$.  
There exists a perfect classifier $h \in \c{R}_2$ if and only if $A$ has no point at $(2i,2i)$ where $i$ is the index of $B$'s variable point.  

Again, this is precisely the indexing problem.  $A$'s points along the diagonal correspond to $n$ bits being $1$ if a point exists and $0$ if not for each index $i$.  And $B$'s index corresponds to the value $i$ of its variable point.  Thus, it requires $\Omega(|D_A|)$ one-way communication, proving the lemma.  
\end{proof}

%\input{new-lin-sep-proof}
%\input{new-lin-sep-highd-proof}

%===========================================================================================================
\bibliographystyle{icml2012}
\bibliography{ml}

\end{document}